\documentclass{article}

\usepackage[preprint,nonatbib]{neurips_2020}

\usepackage{stmaryrd}
\usepackage{natbib}
\usepackage[utf8]{inputenc} 
\DeclareUnicodeCharacter{FB01}{fi}
\usepackage[T1]{fontenc}    
\usepackage{url}            
\usepackage{booktabs}       
\usepackage{amsfonts}       
\usepackage{nicefrac}       
\usepackage{microtype}      
\usepackage{amsmath}
\usepackage{cleveref}
\usepackage{algorithm}
\usepackage{algpseudocode}
\usepackage{graphicx}
\usepackage{multirow}
\usepackage{amssymb,amsthm,amsfonts,amscd,mathrsfs}
\usepackage{thm-restate}

\newtheorem{thm}{Theorem}

\newcommand*{\defeq}{\stackrel{\Delta}{=}}

\title{Scaling Equilibrium Propagation to Deep ConvNets \\
        by Drastically Reducing its Gradient Estimator Bias}

\author{
    {\bfseries Axel Laborieux $^{1, *}$, Maxence Ernoult$^{1,2, *}$,  Benjamin Scellier$^{3, \dagger}$,\vspace{0.2cm}} \\
    {\bfseries Yoshua Bengio$^{3,4}$, Julie Grollier$^2$, Damien Querlioz$^1$ \vspace{0.2cm}}\\
    $^1$Université Paris-Saclay, CNRS, C2N, 91120, Palaiseau, France \\
    $^2$Unité Mixte de Physique, CNRS, Thales, Université Paris-Saclay\\
    $^3$Mila, Université de Montréal\\
    $^4$Canadian Institute for Advanced Research\\
    ${}^\dagger$ Currently at Google\\
    \textsuperscript{*} Corresponding authors: \texttt{\{axel.laborieux, maxence.ernoult\}@c2n.upsaclay.fr}
}

\begin{document}

\maketitle

\begin{abstract}
 Equilibrium Propagation (EP) is a biologically-inspired algorithm for convergent RNNs with a local learning rule that comes with strong theoretical guarantees. The parameter updates of the neural network during the credit assignment phase have been shown mathematically to approach the gradients provided by Backpropagation Through Time (BPTT) when the network is infinitesimally nudged toward its target.  In practice, however, training a network with the gradient estimates provided by EP does not scale to visual tasks harder than MNIST. 
 In this work, we show that a bias in the gradient estimate of EP, inherent in the use of finite nudging, is responsible for this phenomenon and that cancelling it allows training deep ConvNets by EP. We show that this bias can be greatly reduced by using symmetric nudging (a positive nudging and a negative one). We also generalize previous EP equations to the case of cross-entropy loss (by opposition to squared error). As a result of these advances, we are able to achieve a test error of 11.7\% on CIFAR-10 by EP, which approaches the one achieved by BPTT and provides a major improvement with respect to the standard EP approach with same-sign nudging that gives 86\% test error. We also apply these techniques to train an architecture with asymmetric forward and backward connections, yielding a 13.2\% test error. These results highlight EP as a compelling biologically-plausible approach to compute error gradients in deep neural networks.
\end{abstract}

\section{Introduction}
\label{sec:intro}
How synapses in hierarchical neural circuits are adjusted throughout learning a task remains a challenging question called the credit assignment problem \citep{richards2019deep}. 
Equilibrium Propagation (EP) \citep{scellier2017equilibrium} provides a biologically plausible solution to this problem in artificial neural networks. EP is an algorithm for convergent RNNs which, by definition, are given a static input and whose recurrent dynamics converge to a steady state corresponding to the prediction of the network. EP proceeds in two phases, bringing the network to a first steady state, then nudging the output layer of the network towards a ground-truth target until reaching a second steady state. 
During the second phase of EP, the perturbation originating from the output layer propagates through time to upstream layers, creating local error signals that match exactly those that are computed by Backpropagation Through Time (BPTT), the canonical approach for training  RNNs \citep{ernoult2019updates}.
Owing to this strong theoretical guarantee, EP can provide leads for understanding biological learning \citep{lillicrap2020backpropagation}. 
Moreover, the spatial locality of the learning rule prescribed by EP and the possibility to make it also local in time \citep{ernoult2020equilibrium} is highly attractive for designing energy-efficient ``neuromorphic'' hardware implementations of gradient-based learning algorithms \citep{zoppo2020equilibrium,ernoult2020equilibrium}.  

To meet these expectations, however, EP should be able to scale to complex tasks. Until now, works on EP \citep{scellier2017equilibrium,o2018initialized,o2019training,ernoult2019updates,ernoult2020equilibrium} limited their experiments to the MNIST classification task and to shallow network architectures. Despite the theoretical guarantees of EP, the literature suggests that no implementation of EP has thus far succeeded to match the performance of standard deep learning approaches to train deep networks on hard visual tasks. 
This problem is even more challenging when using a more bio-plausible topology where the synaptic connections of the network are asymmetric: existing proposals of EP in this situation \citep{scellier2018generalization, ernoult2020equilibrium} lead to a degradation of accuracy on MNIST compared to standard EP. 
In this work, we show that performing the second phase of EP with nudging strength of constant sign induces a systematic first order bias in the EP gradient estimate which, once cancelled, unlocks the training of deep ConvNets, with symmetric or asymmetric connections, with performance closely matching that of BPTT on CIFAR-10. We also propose to implement the neural network predictor as an external softmax readout. This modification preserves the local nature of EP and  allows us to use the cross-entropy loss, contrary to previous approaches using the squared error loss and where the predictor takes part in the free dynamics of the system. 

Other biologically plausible alternatives to Backpropagation (BP) have attempted to scale to hard vision tasks. \citet{bartunov2018assessing} investigated the use of Feedback Alignment (FA) \citep{lillicrap2016random} and variants of Target Propagation (TP) \citep{lecun1987phd,bengio2014auto} on CIFAR-10 and ImageNet, showing that they perform significantly worse than BP. When the alignment between forward and backward weights is enhanced with extra mechanisms \citep{akrout2019deep}, FA performs better on ImageNet than Sign-Symmetry (SS) \citep{xiao2018biologically}, where feedback weights are taken to be the sign of the forward weights, and almost as well as BP. 
However, in FA and TP, the error feedback does not affect the forward neural activity and is instead routed through a distinct backward pathway, an issue that EP avoids. 
\citet{payeur2020burst} proposed a burst-dependent learning rule that also addresses this problem and whose rate-based equivalent, relying on the use of specialized synapses and complex network topology, has been benchmarked against CIFAR-10 and ImageNet.
In comparison with these approaches, EP offers a minimalistic circuit requirement to handle both inference and gradient computation, which makes it an outstanding candidate for energy-efficient neuromorphic learning hardware design.

More specifically, the contributions of this work\footnote{We post the code at: \url{https://github.com/Laborieux-Axel/Equilibrium-Propagation}} are the following:

\begin{itemize}
    \item We introduce a new method to estimate the gradient of the loss based on three steady states instead of two (section~\ref{sec:improveEstimate}). This approach enables us to achieve 11.68\% test error on CIFAR-10, with 0.6 \% performance degradation only with respect to BPTT. Conversely, we show that using a nudging strength of constant sign yields 86.64\% test error.
    \item We propose to implement the output layer of the neural network as a softmax readout, which subsequently allows us to optimize the cross-entropy loss function with EP. This method improves the classification performance on CIFAR-10 with respect to the use of the squared error loss and is also closer to the one achieved with BPTT (section \ref{sec:readout}). 
    \item Finally, based on ideas of \citet{scellier2018generalization} and \citet{kolen1994backpropagation}, we adapt the learning rule of EP for architectures with distinct (asymmetric) forward and backward connections, yielding only 1.5\% performance degradation on CIFAR-10 compared to symmetric connections (section \ref{sec:vf}).
\end{itemize}

\section{Background}
\label{sec:background}
\subsection{Convergent RNNs With Static Input}

We consider the setting of supervised learning where we are given an input $x$ (e.g., an image) and want to predict a target $y$ (e.g., the class label of that image).
To solve this type of task, Equilibrium Propagation (EP) relies on convergent RNNs, where the input of the RNN at each time step is static and equal to $x$, and the state $s$ of the neural network converges to a steady-state $s_*$.
EP applies to RNNs where the transition function derives from a scalar primitive\footnote{In the original version of EP for real-time dynamical systems \citep{scellier2017equilibrium}, the dynamics derive from an ``energy function'' $E$, which plays a similar role to the primitive function $\Phi$ in the discrete-time setting studied here.} $\Phi$ \citep{ernoult2019updates}.
In this situation, the dynamics of a network with parameters $\theta$ is given by
\begin{equation}
\label{eq:dynamics}
    s_{t+1} = \frac{\partial \Phi}{\partial s}(x, s_{t}, \theta),
\end{equation}
where $s_t$ is the state of the RNN at time step $t$.
After the dynamics have converged\footnote{\citet{scarselli2008graph} discuss sufficient conditions on the transition function to ensure convergence.} at some time step $T$, the network is in the steady state $s_T = s_*$, which, by definition, satisfies:
\begin{equation}
\label{eq:steady}
    s_* = \frac{\partial \Phi}{\partial s}(x, s_*, \theta). 
\end{equation}
The goal of learning is to optimize $\theta$ to minimize the loss $\mathcal{L}^{*} = \ell(s_*, y)$, where $\ell$ is a differentiable cost function.

\subsection{Training Procedures For Convergent RNNs}
\subsubsection{Equilibrium Propagation (EP)}

\citet{scellier2017equilibrium} introduced Equilibrium Propagation in the case of real time dynamics. Subsequent work adapted it to discrete-time dynamics, bringing it closer to conventional deep learning \citep{ernoult2019updates}. EP consists of two distinct phases. During the first (``free'') phase, the RNN evolves according to Eq.~(\ref{eq:dynamics}) for $T$ time steps to ensure convergence to a first steady state $s_*$. During the second (``nudged'') phase of EP, a nudging term $-\beta \frac{\partial \ell}{\partial s}$ is added to the dynamics, with $\beta$ a small scaling factor. 
Denoting $s_{0}^{\beta}$, $s_{1}^{\beta}$, $s_{2}^{\beta}$... the states during the second phase, the dynamics reads
\begin{equation}
\label{eq:secondphase}
    s_{0}^{\beta} = s_*,\quad \mbox{and} \quad \forall t>0, \quad s_{t+1}^{\beta} = \frac{\partial \Phi}{\partial s}(x, s_{t}^{\beta}, \theta) - \beta \frac{\partial \ell}{ \partial s}(s_{t}^{\beta}, y).
\end{equation}
The RNN then reaches a new steady state denoted $s_{*}^{\beta}$.
\citet{scellier2017equilibrium} proposed the update rule, denoting $\eta$ the learning rate applied:
\begin{equation}
\label{eq:estimate}
    \Delta \theta = \eta \widehat{\nabla}^{\rm{EP}}(\beta), \qquad \text{where} \qquad \widehat{\nabla}^{\rm{EP}}(\beta) \defeq \frac{1}{\beta} \left( \frac{\partial \Phi}{\partial \theta}(x, s_{*}^{\beta}, \theta) - \frac{\partial \Phi}{\partial \theta}(x, s_{*}, \theta) \right).
\end{equation}
They proved that this learning rule performs stochastic gradient descent in the limit $\beta \to 0$:
\begin{equation}
\label{eq:formula1}
    \lim_{\beta \to 0} \widehat{\nabla}^{\rm{EP}}(\beta) = - \frac{\partial \mathcal{L}^{*}}{\partial \theta}.
\end{equation}

\subsubsection{Backpropagation Through Time (BPTT)}

The convergent RNNs considered by EP can also be trained by Backpropagation Through Time (BPTT). At each BPTT training iteration, the first phase is performed for $T$ time steps until the network reaches the steady state $s_T = s_*$. The loss at the final time step is computed and the gradients are subsequently backpropagated through the computational graph of the first phase, backward in time. Let us denote $\nabla^{\rm{BPTT}}(t)$ the gradient computed by BPTT truncated to the last $t$ time steps ($T-t, \ldots, T$), which we define formally in Appendix~\ref{sec:BPTT}. A theorem derived by \citet{ernoult2019updates}, inspired from \citet{scellier2019equivalence},  shows that, provided convergence in the first phase has been reached after $T-K$ time steps (i.e., $s_{T-K} = s_{T-K+1} = \ldots = s_T = s_*$), the gradients of EP match those computed by BPTT in the limit $\beta \to 0$, in the first $K$ time steps of the second phase:
\begin{equation}
    \label{eq:integrated}
    \forall t=1,2,\ldots,K,\qquad \widehat{\nabla}^{\rm{EP}}(\beta,t) \defeq
     \frac{1}{\beta} \left( \frac{\partial \Phi}{\partial \theta}(x,s_t^{\beta},\theta) - \frac{\partial \Phi}{\partial \theta}(x,s_*,\theta) \right) \xrightarrow[\beta \to 0]{} \nabla^{\rm{BPTT}}(t).
\end{equation}

\subsection{Convolutional Architectures for Convergent RNNs}
\label{sec:convRNN}

\begin{figure}[ht!]
  \centering
  \includegraphics[width=0.65\textwidth]{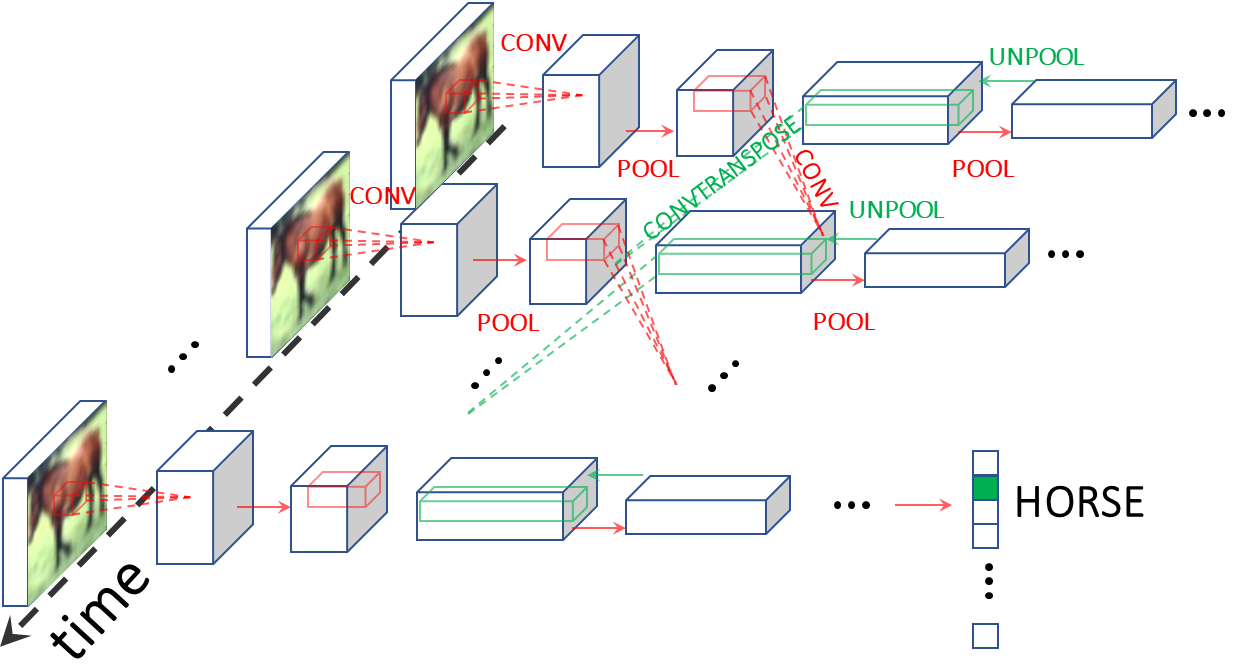}
  \caption{Schematic of the architecture used. We use Equilibrium Propagation (EP) to train a recurrent ConvNet receiving a static input. Red (resp. green) arrows depict forward (resp. backward) operations, with convolutions and transpose convolutions happening through time. At the final time step, the class prediction is carried out. The use of RNNs is inherent in the credit assignment of EP which uses of the temporal variations of the system as error signals for the gradient computation.}
  \label{fig:architecture}
\end{figure}

A convolutional architecture for convergent RNNs with static input was introduced by \citet{ernoult2019updates} and successfully trained with EP on the MNIST dataset. In this architecture, presented in Fig.~\ref{fig:architecture}, we define $N^{\rm conv}$ and $N^{\rm fc}$ the number of convolutional and fully connected layers respectively, and $N^{\rm tot} \defeq N^{\rm conv} + N^{\rm fc}$. $w_{n}$ denotes the weights connecting $s^{n+1}$ to $s^n$, with $s_0 = x$. For notational simplicity here, whether $w_n$ is a convolutional layer or a fully connected layer is implied by the operator, respectively $\star$ for convolutions and $\cdot$ for linear layers. The primitive function can therefore be defined as:
\begin{equation}
\label{eq:phiCNN}
    \Phi(x, \{s^{n}\}) = \sum_{n =0}^{N^{\rm conv}-1} s^{n+1}\bullet\mathcal{P}\left(w_{n+1}\star s^{n}\right) + \sum_{n = N^{\rm conv}}^{N^{\rm tot}-1} s^{n \top}\cdot w_{n+1}\cdot s^{n+1}, 
\end{equation}
where $\bullet$ is the Euclidean scalar product generalized to pairs of tensors with same arbitrary dimension, and $\mathcal{P}$ is a pooling operation. 
Combining Eqs.~(\ref{eq:dynamics}) and~(\ref{eq:phiCNN}), and restricting the space of the state variables to $[0,1]$, yield the dynamics: %
\begin{align}
\label{eq:conv-dynamics}
\left\{
\begin{array}{ll}
s^{n}_{t+1} &= \sigma \left( \mathcal{P}\left(w_{n-1}\star s^{n-1}_t\right) + \tilde{w}_{n+1}\star \mathcal{P}^{-1}\left(s^{n+1}_{t}\right)\right), \qquad 1 \leq n \leq N^{\rm conv} \\
s^n_{t + 1} &= \sigma \left(w_{n-1}\cdot s^{n-1}_{t} + w^{\top}_{n +1}\cdot s^{n+1}_{t} \right), \qquad N^{\rm conv} < n < N^{\rm tot}
\end{array} 
\right.
\end{align}
where $\sigma$ is an activation function bounded between 0 and 1. Transpose convolution and inverse pooling are respectively defined through the convolution by the flipped kernel $\tilde{w}$ and $\mathcal{P}^{-1}$. Plugging Eq.~(\ref{eq:phiCNN}) into Eq.~(\ref{eq:estimate}) yields the learning rule --- see Appendix~\ref{sec:appConvArch} for implementation details.

\subsection{Equilibrium Propagation with asymmetric synaptic connections}
\label{sec:vf}

In the standard formulation of EP, the dynamics of the neural network derive from a function $\Phi$ (Eq.~(\ref{eq:dynamics})) called the primitive function. For better biological plausibility, subsequent works have proposed a more general formulation of EP which circumvents this requirement and allows training networks with distinct (asymmetric) forward and backward connections \citep{scellier2018generalization,ernoult2020equilibrium}. In this setting, the dynamics of Eq.~(\ref{eq:dynamics}) is changed into the more general form: 
\begin{equation}
\label{eq:general-dynamics}
    s_{t+1} = F(x, s_t, \theta),
\end{equation}
and the conventionally proposed learning rule reads:
\begin{equation}
\label{eq:vf-estimate}
    \Delta \theta = \eta \widehat{\nabla}^{\rm{VF}}(\beta), \qquad \text{where} \qquad \widehat{\nabla}^{\rm{VF}}(\beta) \defeq \frac{1}{\beta} \frac{\partial F}{\partial \theta}(x, s_*, \theta)^\top \cdot \left(s^\beta_* - s_*\right),
\end{equation}
where VF stands for Vector Field \citep{scellier2018generalization}. If the transition function $F$ derives from a primitive function $\Phi$ (i.e., if $F = \frac{\partial \Phi}{\partial s}$), then $\widehat{\nabla}^{\rm VF}(\beta)$ is equal to $\widehat{\nabla}^{\rm EP}(\beta)$ in the limit $\beta \to 0$ ( i.e. $\lim_{\beta \to 0}\widehat{\nabla}^{\rm VF}(\beta) = \lim_{\beta \to 0}\widehat{\nabla}^{\rm EP}(\beta)$).

\section{Improving EP Training}
\label{sec:improveEP}

\citet{ernoult2019updates} showed that the temporal variations of the network over the second phase of EP exactly compute BPTT gradients in the limit $\beta \to 0$ (Eq.~(\ref{eq:integrated})). This result appears to underpin the use of two phases as a fundamental element of EP, but is it really the case? In this section, we revisit EP as a gradient estimation procedure and propose an implementation in three phases instead of two. Moreover, we show how to optimize the cross-entropy loss function with EP. Combining these two new techniques enabled us to achieve the best performance on CIFAR-10 by EP, on architectures with symmetric and asymmetric forward and backward connections.

\subsection{Reducing bias and variance in the gradient estimate of the loss function}
\label{sec:improveEstimate}

\begin{figure}[ht!]
  \centering
  \includegraphics[width=0.65\textwidth]{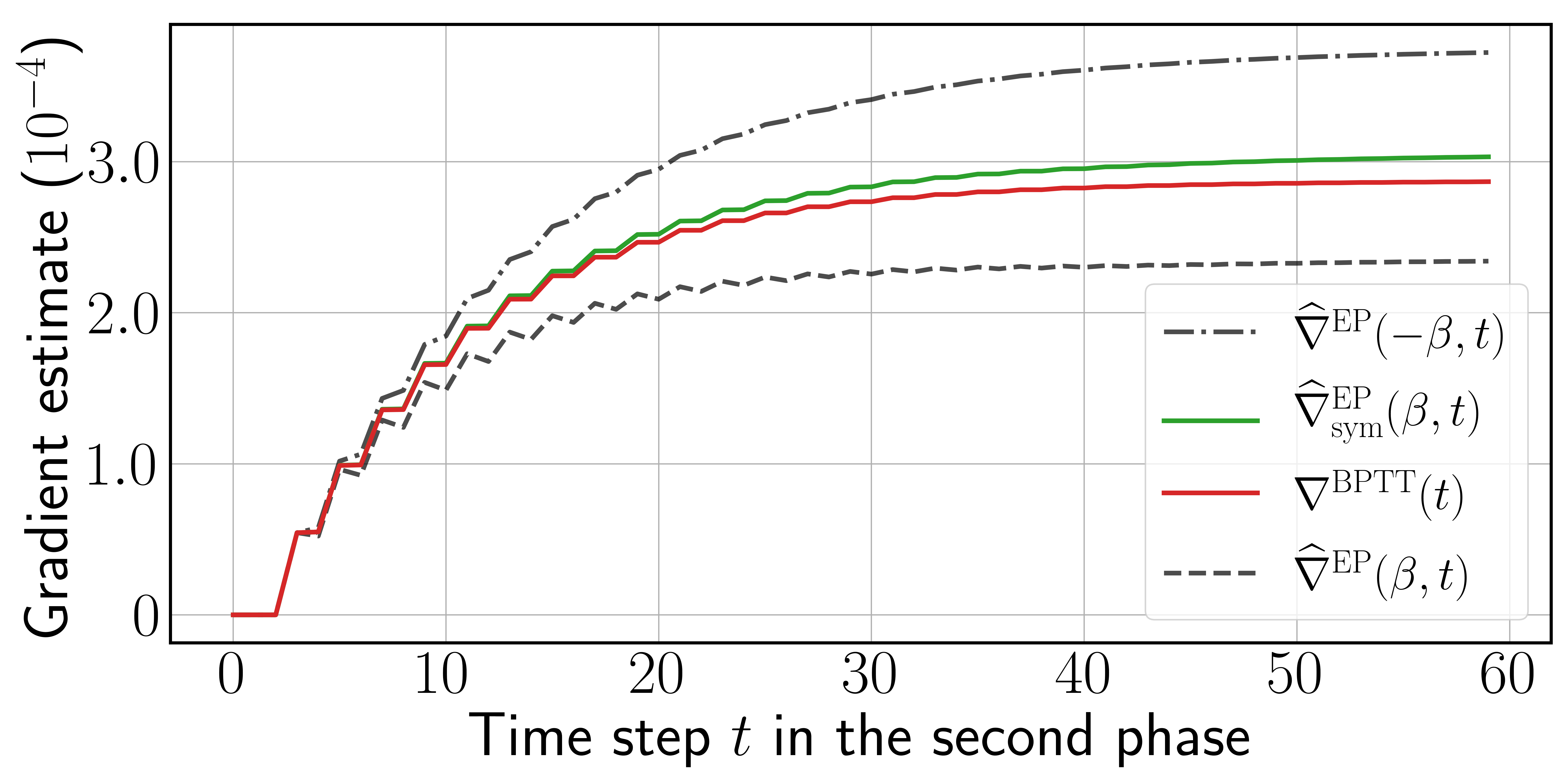}
  \caption{One-sided EP gradient estimate for opposite values of $\beta = 0.1$ (black dashed curves), symmetric EP gradient estimate (green curve) and reference gradients computed by BPTT (red curve) computed over the second phase, for a single weight chosen at random. The time step $t$ is defined for BPTT and EP according to Eq.~(\ref{eq:integrated}). More instances can be found in Appendix \ref{sec:appCompEstimate}.
  }
  \label{fig:thirdphase}
\end{figure}

In the first formulation of EP, \citet{scellier2017equilibrium} state their result as: 
\begin{equation}
\label{eq:formula}
    \left. \frac{d}{d\beta} \right|_{\beta=0} \frac{\partial \Phi}{\partial \theta}(x, s_{*}^{\beta}, \theta) = - \frac{\partial \mathcal{L}^{*}}{\partial \theta}.
\end{equation}
In its original implementation, EP evaluates the left-hand side of Eq.~(\ref{eq:formula}) using the estimate $\widehat{\nabla}^{\rm{EP}}(\beta)$ with two points $\beta = 0$ and $\beta > 0$, thereby calling for the need of two phases. However, the use of $\beta > 0$ in practice induces a systematic first order bias in the gradient estimation provided by EP. In order to eliminate this bias, we propose to perform a third phase with $-\beta$ as the nudging factor, keeping the first and second phases unchanged. We then estimate the gradient of the loss using the following symmetric difference estimate:
\begin{equation}
\label{eq:thirdphase}
\widehat{\nabla}^{\rm{EP}}_{\rm sym}(\beta) \defeq \frac{1}{2\beta} \left( \frac{\partial \Phi}{\partial \theta}(x, s_{*}^{\beta}, \theta) - \frac{\partial \Phi}{\partial \theta}(x, s_{*}^{-\beta}, \theta) \right).
\end{equation}
Indeed, under mild assumptions on the function $\beta \mapsto \frac{\partial \Phi}{\partial \theta}(x,s_{*}^{\beta}, \theta)$, we can show that, as $\beta \to 0$:

\begin{align}
\widehat{\nabla}^{\rm{EP}}(\beta) &= - \frac{\partial \mathcal{L}^{*}}{\partial \theta} + \frac{\beta}{2} \left. \frac{d^2}{d\beta^2} \right|_{\beta=0} \frac{\partial \Phi}{\partial \theta}(s_{*}^{\beta}, \theta) + O(\beta^2), \label{eq:estimate-EP}\\
\widehat{\nabla}^{\rm{EP}}_{\rm sym}(\beta) &= - \frac{\partial \mathcal{L}^{*}}{\partial \theta} + O(\beta^2).\label{eq:estimate-EP-sym}
\end{align}
This result is proved in Lemma \ref{lma:lemma} of Appendix \ref{sec:DL}. Eq.~(\ref{eq:estimate-EP}) shows that the estimate $\widehat{\nabla}^{\rm{EP}}(\beta)$ possesses a first-order error term in $\beta$ which the symmetric estimate $\widehat{\nabla}^{\rm{EP}}_{\rm sym}(\beta)$ eliminates (Eq.~(\ref{eq:estimate-EP-sym})).
Note that the first-order term of $\widehat{\nabla}^{\rm{EP}}(\beta)$ could also be cancelled out on average by choosing the sign of $\beta$ at random with even probability (so that $\mathbb{E}(\beta) = 0$, see Alg.~\ref{alg:rndsign} of Appendix \ref{sec:appendOneSided}). 
Although not explicitly stated in this purpose, the use of such randomization has been reported in some earlier publications on the MNIST task \citep{scellier2017equilibrium,ernoult2020equilibrium}. However, in this work, we show that this method exhibits high variance in the training procedure. 
We call $\widehat{\nabla}^{\rm{EP}}(\beta)$ and $\widehat{\nabla}^{\rm{EP}}_{\rm sym}(\beta)$ the one-sided and symmetric EP gradient estimates respectively. The qualitative difference between these estimates is depicted on Fig.~\ref{fig:thirdphase} and the full training procedure is depicted in Alg.~\ref{alg:thirdphase} of Appendix \ref{sec:appendthirdphase}. Finally, this technique can also be applied to the Vector Field setting introduced in section~\ref{sec:vf} and we denote $\widehat{\nabla}^{\rm{VF}}_{\rm sym}(\beta)$ the resulting symmetric estimate --- see Appendix \ref{sec:appConvAsym} for details.

\subsection{Changing the loss function}
\label{sec:readout}

We introduce a novel architecture to optimize the cross-entropy loss with EP, narrowing the gap with conventional deep learning architectures for classification tasks. In the next paragraph, we denote $\widehat{y}$ the set of neurons that carries out the prediction of the neural network.

\begin{figure}[ht!]
  \centering
  \includegraphics[width=0.85\textwidth]{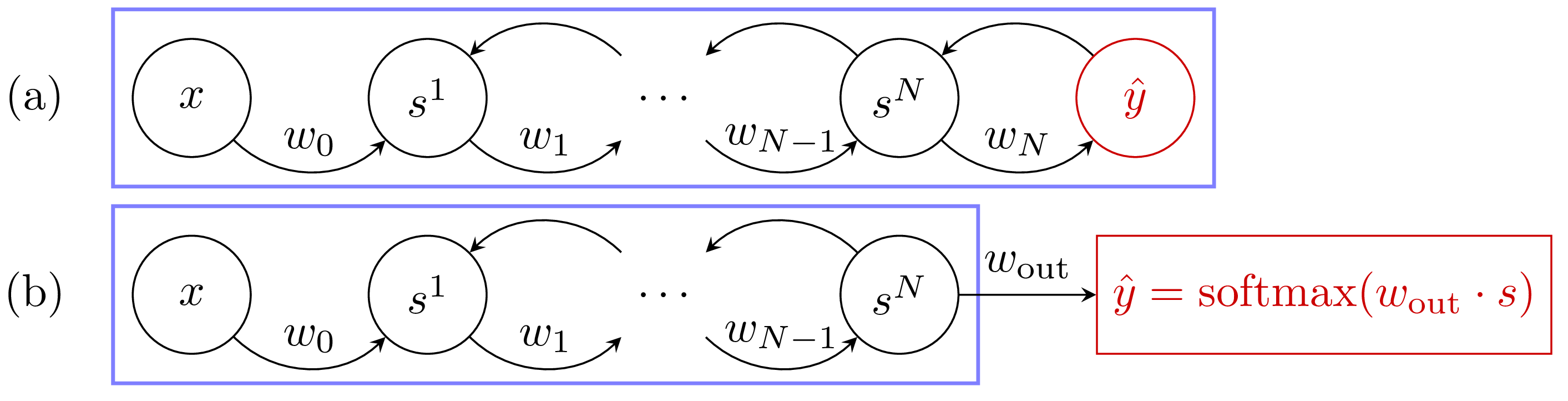}
  \caption{Free dynamics of the architectures used for the two loss functions where the blue frame delimits the system. \textbf{(a) Squared Error loss function.} The usual setting where the predictor $\hat{y}$ (in red) takes part in the free dynamics of the neural network through bidirectional synaptic connections. \textbf{(b)~Cross-Entropy loss function.} The new approach proposed in this work where the predictor $\hat{y}$ (also in red) is no longer involved in the system free dynamics and is implemented as a softmax readout.  
  }
  \label{fig:archi}
\end{figure}

\paragraph{Squared Error loss function.}
Previous implementations of EP used the squared error loss. Using this loss function for EP is natural, as in this setting, the output $\widehat{y}$ is viewed as a part of $s$ (the state variable of the network), which can influence the state of the network through bidirectional synaptic connections (see Fig.~\ref{fig:archi}). The state of the network is of the form $s=(s^1, \dots, s^N,\widehat{y})$ where $h=(s^1, \dots, s^N)$ represent the ``hidden layers'', and the corresponding cost function is
\begin{equation}
    \ell(\widehat{y},y) = \frac{1}{2} \left\| \widehat{y} - y \right\|^2.
\end{equation}

The second phase dynamics of the hidden state and output layer given by Eq.~(\ref{eq:secondphase}) read, in this context:
\begin{equation}
    \label{eq:elastic}
    h_{t+1}^{\beta} = \frac{\partial \Phi}{\partial h}(x, h_t^{\beta}, \widehat{y}_t^{\beta}, \theta), \qquad \widehat{y}_{t+1}^{\beta} = \frac{\partial \Phi}{\partial \widehat{y}}(x, h_t^{\beta}, \widehat{y}_t^{\beta}, \theta) + \beta \; (y - \widehat{y}_t^{\beta}).
\end{equation}

\paragraph{Softmax readout, Cross-Entropy loss function.}
In this paper, we propose an alternative approach, where the output $\widehat{y}$ is not a part of the state variable $s$ but is instead implemented as a read-out (see Fig.~\ref{fig:archi}), which is a function of $s$ and of a weight matrix $w_{\rm out}$ of size $\dim(y) \times \dim(s)$. In practice, $w_{\rm out}$ reads out the last convolutional layer. At each time step $t$ we define:
\begin{equation}
    \widehat{y}_t = \mbox{softmax}(w_{\rm out}\cdot s_t).
\end{equation}
The cross-entropy cost function associated with the softmax readout is then:
\begin{equation}
\label{eq:cross-entropy}
\ell(s,y,w_{\rm out}) = - \sum_{c=1}^C y_c \log(\textrm{softmax}_c(w_{\rm out} \cdot s)).
\end{equation}
Using $\frac{\partial \ell}{\partial s}(s,y,w_{\rm out}) = w_{\rm out}^\top \cdot \left( \textrm{softmax}(w_{\rm out} \cdot s) - y \right)$, the second phase dynamics given by Eq.~(\ref{eq:secondphase}) read in this context:
\begin{equation}
s_{t+1}^{\beta} = \frac{\partial \Phi}{\partial s}(x, s_{t}^{\beta}, \theta) + \beta \; w_{\rm out}^\top \cdot \left( y - \widehat{y}_t^\beta \right).
\end{equation}
Note here that the loss $\mathcal{L}^* = \ell(s_*,y,w_{\rm out})$ also depends on the parameter $w_{\rm out}$. Appendix~\ref{sec:appConvSym-CE} provides the learning rule applied to $w_{\rm out}$.

\subsection{Changing the learning rule of EP with asymmetric synaptic connections}
\label{sec:new-vf}

In the case of architectures with asymmetric connections, applying the traditional EP learning rule directly, as given by Eq.~(\ref{eq:vf-estimate}), prescribes different forward and backward weights updates, resulting in significantly different forward and backward weights throughout learning. However, the theoretical equivalence between EP and BPTT only holds for symmetric connections. Until now, training experiments of asymmetric weights EP have performed worse than symmetric weights EP \citep{ernoult2020equilibrium}. In this work, therefore, we tailor a new learning rule for asymmetric weights, described in detail Appendix~\ref{sec:appConvAsym}, where the forward and backward weights undergo the same weight updates, incorporating an equal leakage term. This way, forward and backward weights, although they are independently initialized, naturally converge to identical values throughout the learning process. A similar methodology, adapted from \citet{kolen1994backpropagation}, has been shown to improve the performance of Feedback Alignment in Deep ConvNets \citep{akrout2019deep}. 

Assuming general dynamics of the form of Eq.~(\ref{eq:general-dynamics}), we distinguish forward connections $\theta_{\rm f}$ from backward connections $\theta_{\rm b}$ so that $\theta = \{\theta_{\rm f}, \theta_{\rm b}\}$, with $\theta_{\rm f}$ and $\theta_{\rm b}$ having same dimension. Assuming a first phase, a second phase with $\beta > 0$ and a third phase with $-\beta$, we define:
\begin{equation}
\forall {\rm i \in \{ f,b \} }, \qquad \overline{\nabla_{\theta_{\rm i}}^{\rm VF}}(\beta) = \frac{1}{2 \beta} \left( \frac{\partial F}{\partial \theta_{\rm i}}^\top(x, s_*^\beta, \theta) \cdot s_*^\beta - \frac{\partial F}{\partial \theta_{\rm i}}^\top(x, s_*^{-\beta}, \theta)\cdot s_*^{-\beta} \right)
\end{equation}
and we propose the following update rules:
\begin{align}
\label{eq:new-vf}
\left\{
\begin{array}{l}
\displaystyle \Delta \theta_{\rm f} = \eta \left(\widehat{\nabla}^{\rm{KP-VF}}_{\rm sym}(\beta) - \lambda \theta_{\rm f}\right)\\
\displaystyle \Delta \theta_{\rm b} = \eta \left( \widehat{\nabla}^{\rm{KP-VF}}_{\rm sym}(\beta) - \lambda \theta_{\rm b}\right)
\end{array}
\right.,
\quad \mbox{with} \quad \widehat{\nabla}^{\rm{KP-VF}}_{\rm sym}(\beta) = \frac{1}{2}(\overline{\nabla_{\theta_{\rm f}}^{\rm VF}}(\beta)+\overline{\nabla_{\theta_{\rm b}}^{\rm VF}}(\beta))
\end{align}
where $\eta$ is the learning rate and $\lambda$ a leakage parameter. The estimate $\widehat{\nabla}^{\rm{KP-VF}}_{\rm sym}(\beta)$ can be thought of a generalization of Eq.~(\ref{eq:thirdphase}), as highlighted in Appendix \ref{sec:appConvAsym} with an explicit application of Eq.~(\ref{eq:new-vf}) to a ConvNet.

\section{Experimental results}

In this section, we implement EP with the modifications described in section \ref{sec:improveEP} and successfully train deep ConvNets on the CIFAR-10 vision task \citep{krizhevsky2009learning}. Our architectures consist of four convolutional layers of $3 \times 3$ kernels each followed by max pooling layers, and one fully connected layer. Implementation and model details are provided in Appendix~\ref{sec:appExpDetail} and \ref{sec:appConvArch} respectively.

\subsection{ConvNets with symmetric connections}
\begin{table}[ht!] 
\caption{Performance comparison on CIFAR-10 between BPTT and EP with several gradient estimation schemes. Note that the different gradient estimates only apply to EP. We indicate over ﬁve trials the mean and standard deviation in parenthesis for the test error, and the mean train error.}
\label{tab:results}
\centering
\begin{tabular}{llllll}
\cline{2-6}
                               & \multicolumn{3}{c}{Equilibrium Propagation Error (\%)}                                                 & \multicolumn{2}{c}{BPTT Error (\%)}                                  \\ \cline{2-6} 
\multicolumn{1}{c}{}           & \multicolumn{1}{c}{Gradient Estimate} & \multicolumn{1}{c}{Test}           & \multicolumn{1}{c}{Train} & \multicolumn{1}{c}{Test}                 & \multicolumn{1}{c}{Train} \\ \hline
\multirow{3}{*}{Squared Error} & One-sided                             & $86.64$ $(5.82)$                   & $84.90$                   & \multirow{3}{*}{$11.10$ $(0.21)$} & \multirow{3}{*}{$3.69$}   \\
                               & Random Sign                           & $21.55$ $(20.00)$                  & $20.01$                   &                                          &                           \\
                               & Symmetric                             & $12.45$  $(0.18)$                  & $7.83$                    &                                          &                           \\ \hline
Cross-Ent.                     & Symmetric                             & $\mathbf{11.68}$ $\mathbf{(0.17)}$ & $\mathbf{4.98}$           & $11.12$ $(0.21)$                         & $2.19$                    \\
Cross-Ent. (Dropout)           & Symmetric                             & $11.87$ $(0.29)$                   & $6.46$                    & $10.72$ $(0.06)$                         & $2.99$                    \\ \hline
\end{tabular}
\end{table}

We first consider the setting of section~\ref{sec:convRNN}, where bottom-up and top-down signals are carried by the same set of weights.
In Table~\ref{tab:results}, we compare the performance achieved by the ConvNet for each EP gradient estimate introduced in section~\ref{sec:improveEstimate} with the performance achieved by BPTT.

The one-sided gradient estimate leads to unstable training behavior where the network is unable to fit the data. When the bias in the gradient estimate is averaged out by choosing at random the sign of $\beta$ during the second phase, the average test error over five runs goes down to $21.55\%$. However, one run among the five yielded instability similar to the one-sided estimate, whereas the four remaining runs lead to $12.61\%$ test error and $8.64 \%$ train error. This method for estimating the loss gradient thus presents high variance --- further experiments shown in Appendix \ref{sec:appExpDetail} confirm this tendency. Conversely, the symmetric estimate enables EP to consistently reach $12.45\%$ test error, with only $1.35\%$ degradation with respect to BPTT. Therefore, removing the first-order error term in the gradient estimate is critical for scaling to deeper architectures. However, proceeding to this end deterministically (with three phases) rather than stochastically (with a randomized nudging sign) seems to be more reliable. 

The readout scheme introduced in section~\ref{sec:readout} to optimize the cross-entropy loss function enables EP to narrow the performance gap with BPTT down to $0.56\%$ while outperforming the Squared Error setting by $0.77\%$. Finally, we adapted dropout \citep{srivastava2014dropout} to convergent RNNs (see Appendix~\ref{sec:dropout} for implementation details) to see if the performance could be improved further. However, we can observe from Table~\ref{tab:results} that contrary to BPTT, the EP test error is not improved by adding a $0.1$ dropout probability in the neuron layer after the convolutions.

\subsection{ConvNets with asymmetric connections}
We investigate here the performance achieved by EP when the architecture uses distinct forward and backward weights, using a softmax readout. We find that the estimate $\widehat{\nabla}^{\rm{VF}}_{\rm sym}(\beta)$ leads to a poor performance with $75.47\%$ test-error and concomitantly observe that forward and backward weight do not align well. Conversely, when using our new estimate $\widehat{\nabla}^{\rm{KP-VF}}_{\rm sym}(\beta)$ defined in section~\ref{sec:new-vf}, a good performance is recovered with $1.5 \%$ performance degradation with respect to the architecture with symmetric connections, and a $3\%$ degradation with respect to BPTT (see Table~\ref{tab:asym}). In this case, forward and backward weights are perfectly aligned by epoch 50, as observed in the weight alignment curves in Fig.~\ref{fig:angle} of Appendix~\ref{sec:appAngle}. These results suggest that enhancing forward and backward weights alignment might also help EP training in deep ConvNets. 

\begin{table}[ht!]
\caption{CIFAR-10 results for asymmetric connections.}
\label{tab:asym}
\centering
\begin{tabular}{lcllll}
\cline{2-6}
                               & \multicolumn{3}{c}{Equilibrium Propagation Error (\%)}                                                   & \multicolumn{2}{c}{BPTT Error (\%)}                          \\ \cline{2-6} 
\multicolumn{1}{c}{}           & Gradient Estimate                       & \multicolumn{1}{c}{Test}           & \multicolumn{1}{c}{Train} & \multicolumn{1}{c}{Test}         & \multicolumn{1}{c}{Train} \\ \hline
\multirow{2}{*}{Cross-Entropy} & $\widehat{\nabla}^{\rm{VF}}_{\rm sym}$ & $75.47$ $(4.72)$                   & $78.04$                   & \multirow{2}{*}{$9.46$ $(0.17)$} & \multirow{2}{*}{$0.80$}   \\
                               & $\widehat{\nabla}^{\rm{KP-VF}}_{\rm sym}$   & $\mathbf{13.15}$ $\mathbf{(0.49)}$ & $8.87$                    &                                  &                           \\ \hline
\end{tabular}
\end{table}

\section{Discussion}
\label{sec:discussion}

In comparison with conventional implementations of EP, our results unveil the necessity to compute better gradient estimates in order to scale EP to deep ConvNets on hard visual tasks. Keeping the first order gradient estimate bias of EP, as done traditionally, severely impedes the training of these architectures and, conversely, removing it brings EP performance on CIFAR-10 close to the one achieved by BPTT.
While the test accuracy of our adapted EP and BPTT are very close, we can remark in Table~\ref{tab:results} that BPTT fits the training data better  than EP by at least $2.8\%$, and that dropout only improves BPTT performance. These two combined insights suggest that EP training may have a self-regularizing effect, similar to the effects of dropout, which we hypothesize to be due to the residual estimation bias of the BPTT gradients by EP.

Employing a new training technique that preserves the spatial locality of EP computations,  our results extend to the case of an architecture with distinct forward and backward synaptic connections. We  only observe a $1.5\%$ performance degradation with respect to the symmetric architecture. This result demonstrates the scalability of EP without the biologically implausible requirement of a symmetric connectivity pattern.

Our three steady states-based gradient estimate comes at a computational cost since one more phase is needed with regards to the conventional implementation. Even though the steady state of the free phase $s_*$ is not used to compute the gradient estimate in Eq.~(\ref{eq:thirdphase}), we experimentally find that $s_*$ is needed as a starting point for the second and third phases. 

In the longer run, the full potential of EP will be best envisioned on neuromorphic hardware \citep{ernoult2019updates, ernoult2020equilibrium, zoppo2020equilibrium}, which can sustain fast analog device physics and use them to implement the dynamics of EP intrinsically \citep{romera2018vowel, ambrogio2018equivalent}. Our prescription to run two nudging phases with opposite nudging strengths could be naturally implemented on such systems, which often function differentially to cancel device inherent biases \citep{bocquet2018memory}. Overall, our work provides evidence that EP is one compelling approach to scale neuromorphic on-chip training to real-world tasks in a fully local fashion.

\section*{Broader Impact}
This work may have a long-term impact on the design of energy-efficient hardware leveraging the physics of the device to perform learning. The demonstration that EP can scale to deep networks may also provide insights to neuroscientists to understand the mechanisms of credit assignment in the brain.
Due to the long term nature of this impact, the positive and negative outcomes of this work cannot yet be stated.

\section*{Acknowledgements}
The authors would like to thank Thomas Fischbacher for useful feedback and discussions. This work was supported by European Research Council Starting Grant NANOINFER (reference: 715872), European Research Council Grant bioSPINspired (reference: 682955), CIFAR, NSERC and Samsung.

\bibliographystyle{abbrvnat}
\bibliography{biblio}

\newpage
\appendix
\part*{Appendix}

\section{Gradients of BPTT}
\label{sec:BPTT}

In this appendix, we define $\nabla^{\rm{BPTT}}(t)$, the gradient computed by BPTT truncated to the last $t$ time steps ($T-t, \ldots, T$). To do this, let us rewrite Eq.~(\ref{eq:dynamics}) as $s_{t+1} = \frac{\partial \Phi}{\partial s}(x, s_{t}, \theta_t = \theta)$, where $\theta_t$ denotes the parameter at time step $t$, the value $\theta$ being shared across all time steps. We consider the loss after $T$ time steps $\mathcal{L} = \ell(s_T,y)$. Rewriting the dynamics in such a way enables us to define $\frac{\partial \mathcal{L}}{\partial \theta_t}$ as the sensitivity of the loss with respect to $\theta_t$, when $\theta_0, \ldots, \theta_{t-1}, \theta_{t+1}, \ldots, \theta_{T-1}$ remain fixed (set to the value $\theta$). With these notations, the gradient computed by BPTT truncated to the last $t$ time steps is
\begin{equation}
    \nabla^{\rm{BPTT}}(t) = \frac{\partial \mathcal{L}}{\partial \theta_{T-t}} + \ldots + \frac{\partial \mathcal{L}}{\partial \theta_{T-1}}.
\end{equation}

\section{Error terms in the estimates of the loss gradient}
\label{sec:DL}

In this appendix, we prove Lemma \ref{lma:lemma} which shows that $\widehat{\nabla}^{\rm{EP}}_{\rm sym}(\beta)$ is a better estimate of $- \frac{\partial \mathcal{L}^{*}}{\partial \theta}$ than $\widehat{\nabla}^{\rm{EP}}(\beta)$. First, we recall the theorem proved in \citet{scellier2017equilibrium}.
\begin{thm}[\citet{scellier2017equilibrium}]
\begin{equation}
\left. \frac{d}{d\beta} \right|_{\beta=0} \frac{\partial \Phi}{\partial \theta}(x, s_{*}^{\beta}, \theta) = - \frac{\partial \mathcal{L}^{*}}{\partial \theta}.
\end{equation}
\label{thm:main}
\end{thm}

We also recall that the two estimates (one-sided and symmetric) are, by definition:
\begin{align*}
\widehat{\nabla}^{\rm{EP}}(\beta) & \defeq \frac{1}{\beta} \left( \frac{\partial \Phi}{\partial \theta}(x, s_{*}^{\beta}, \theta) - \frac{\partial \Phi}{\partial \theta}(x, s_{*}, \theta) \right), \\
\widehat{\nabla}^{\rm{EP}}_{\rm sym}(\beta) & \defeq \frac{1}{2\beta} \left( \frac{\partial \Phi}{\partial \theta}(x, s_{*}^{\beta}, \theta) - \frac{\partial \Phi}{\partial \theta}(x, s_{*}^{-\beta}, \theta) \right).
\end{align*}

Finally we recall Lemma \ref{lma:lemma}, for readability.

\begin{restatable}{lma}{estimates}
\label{lma:lemma}
Provided the function $\beta \mapsto \frac{\partial \Phi}{\partial \theta}(x,s_{*}^{\beta}, \theta)$ is three times differentiable, we have, as $\beta \to 0$:
\begin{align*}
\widehat{\nabla}^{\rm{EP}}(\beta) &= - \frac{\partial \mathcal{L}^{*}}{\partial \theta} + \frac{\beta}{2} \left. \frac{d^2}{d\beta^2} \right|_{\beta=0} \frac{\partial \Phi}{\partial \theta}(s_{*}^{\beta}, \theta) + O(\beta^2),\\
\widehat{\nabla}^{\rm{EP}}_{\rm sym}(\beta) &= - \frac{\partial \mathcal{L}^{*}}{\partial \theta} + O(\beta^2).
\end{align*}
\end{restatable}

\begin{proof}[Proof of Lemma \ref{lma:lemma}]
Let us define
\begin{equation*}
    f(\beta) \defeq \frac{\partial \Phi}{\partial \theta}(x,s_{*}^{\beta}, \theta).
\end{equation*}
The formula of Theorem \ref{thm:main} rewrites
\begin{equation*}
    f'(0) = - \frac{\partial \mathcal{L}^{*}}{\partial \theta}.
\end{equation*}
As $\beta \to 0$, we have the Taylor expansion
\begin{equation}
\label{eq:taylor:pos}
f(\beta) = f(0) + \beta f'(0) + \frac{\beta^2}{2} f''(0) +  O(\beta^3).
\end{equation}
With these notations, the one-sided estimate reads
\begin{align*}
\widehat{\nabla}^{\rm{EP}}(\beta) & = \frac{1}{\beta} \left( f(\beta) - f(0) \right) \\
& = f'(0) + \frac{\beta}{2} f''(0) + O(\beta^2) \\
& = - \frac{\partial \mathcal{L}^{*}}{\partial \theta} + \frac{\beta}{2} \left. \frac{d^2}{d\beta^2} \right|_{\beta=0} \frac{\partial \Phi}{\partial \theta}(x, s_{*}^{\beta}, \theta)  + O(\beta^2).
\end{align*}

We can also write a Taylor expansion around $0$ at the point $-\beta$. We have
\begin{equation}
\label{eq:taylor:neg}
f(-\beta) = f(0) - \beta f'(0) + \frac{\beta^2}{2} f''(0) +  O(\beta^3).
\end{equation}
Subtracting Eq.~\ref{eq:taylor:neg} from Eq.~\ref{eq:taylor:pos}, we can rewrite the symmetric difference estimate as
\begin{align*}
\widehat{\nabla}^{\rm{EP}}_{\rm sym}(\beta) & = \frac{1}{2 \beta} \left( f(\beta) - f(-\beta) \right) \\
& = f'(0) +O(\beta^2) \\
& = - \frac{\partial \mathcal{L}^{*}}{\partial \theta} + O(\beta^2).
\end{align*}
The derivative to the third order of $f$ is only used to get the $O(\beta^3)$ term in the expansion Eq.~(\ref{eq:taylor:pos}), it can be changed into $o(\beta^2)$ if we only assume $f$ twice differentiable.
\end{proof}

\section{Pseudo code}
\label{sec:appCode}

\subsection{Random one-sided estimation of the loss gradient}
\label{sec:appendOneSided}

In this appendix, we define the random one-sided estimation used in this work and by \citet{scellier2017equilibrium,ernoult2020equilibrium}.

\begin{algorithm}[H]{\emph{Input}: $x$, $y$, $\theta$, $\eta$.  \\
\emph{Output}: $\theta$.}
    \caption{EP with random one-sided estimation of the loss gradient. We omit the activation function $\sigma$ for clarity.}\label{alg:rndsign}
    \begin{algorithmic}[1]
        \State $s_0 \gets 0$
        \For{$t=0$ to $T$} \Comment{First phase.}
        \State $s_{t+1} \gets \frac{\partial \Phi}{\partial s} (x, s_t, \theta)$
        \EndFor
        \State $s_* \gets s_T$
        \State $\beta \gets \beta \times  {\rm{Bernoulli}(1, -1)}$ \Comment{Random sign.}
        \State $s_{0}^{\beta} \gets s_*$
        \For{$t=0$ to $K$} \Comment{Second phase.}
        \State $s_{t+1}^{\beta} \gets \frac{\partial \Phi}{\partial s} (x, s_t^\beta, \theta) - \beta \frac{\partial \ell}{\partial s}(s_{t}^{\beta}, y)$
        \EndFor
        \State $s_{*}^{\beta} \gets s_{K}^{\beta}$
        \State $\nabla_{\theta}^{\rm{EP}} \gets \frac{1}{\beta} \left( \frac{\partial \Phi}{\partial \theta}(s_{*}^{\beta}, \theta) - \frac{\partial \Phi}{\partial \theta}(s_{*}, \theta) \right)$
        \State $\theta \gets \theta + \eta \nabla_{\theta}^{\rm{EP}}$ \\
        \Return $\theta$
    \end{algorithmic}
\end{algorithm}

\subsection{Symmetric difference estimation of the loss gradient}
\label{sec:appendthirdphase}

In this appendix, we define the estimation procedure using a symmetric difference estimate introduced in this work.

\begin{algorithm}[H]{\emph{Input}: $x$, $y$, $\theta$, $\eta$.  \\
\emph{Output}: $\theta$.}
    \caption{EP with symmetric difference estimation of the loss gradient. We omit the activation function $\sigma$ for clarity.}\label{alg:thirdphase}
        \begin{algorithmic}[1]
        \State $s_0 \gets 0$
        \For{$t=0$ to $T$} 
        \State $s_{t+1} \gets \frac{\partial \Phi}{\partial s} (x, s_t, \theta)$ \Comment{First phase.}
        \EndFor
        \State $s_* \gets s_T$ \Comment{Store the free steady state.}
        \State $s_{0}^{\beta} \gets s_*$ 
        \For{$t=0$ to $K$} 
        \State $s_{t+1}^{\beta} \gets \frac{\partial \Phi}{\partial s} (x, s_{t}^{\beta}, \theta) - \beta \frac{\partial \ell}{\partial s}(s_{t}^{\beta}, y)$ \Comment{Second phase.}
        \EndFor
        \State $s_{*}^{\beta} \gets s_{K}^{\beta}$
        \State $s_{0}^{-\beta} \gets s_*$ \Comment{Back to the free steady state.}
        \For{$t=0$ to $K$}
        \State $s_{t+1}^{-\beta} \gets \frac{\partial \Phi}{\partial s} (x, s_{t}^{-\beta}, \theta) + \beta \frac{\partial \ell}{\partial s}(s_{t}^{-\beta}, y)$ \Comment{Third phase.}
        \EndFor
        \State $s_{*}^{-\beta} \gets s_{K}^{-\beta}$
        \State $\widehat{\nabla}_{\theta}^{\rm{EP}} \gets \frac{1}{2\beta} \left( \frac{\partial \Phi}{\partial \theta}(s_{*}^{\beta}, \theta) - \frac{\partial \Phi}{\partial \theta}(s_{*}^{-\beta}, \theta) \right)$
        \State $\theta \gets \theta + \eta \widehat{\nabla}_{\theta}^{\rm{EP}}$ \\
        \Return $\theta$
    \end{algorithmic}
\end{algorithm}

\section{Convolutional RNNs}
\label{sec:appConvArch}

Throughout this section, $N^{\rm conv}$ and $N^{\rm fc}$ denote respectively the number of convolutional layers and fully connected layers in the convolutional RNN, and $N^{\rm tot} \defeq N^{\rm conv} + N^{\rm fc}$. The neuron layers are denoted by $s$ and range from $s^{0}=x$ the input to the output $s^{N^{\rm tot}}$ in the case of squared error, or $s^{N^{\rm tot}-1}$ in the case of softmax read-out. 

\subsection{Definition of the operations}
\label{sec:appConvOP}
In this subsection we detail the operations involved in the dynamics of a convolutional RNN.

\begin{itemize}
    \item The 2-D convolution between $w$ with dimension $(C_{\rm out}, C_{\rm in}, F,F)$ and an input $x$ of dimensions $(C_{\rm in}, H_{\rm in}, W_{\rm in})$ and stride one is a tensor $y$ of size $(C_{\rm out}, H_{\rm out}, W_{\rm out})$ defined by:
    \begin{equation}
        y_{c,h,w} = (w \star x)_{c,h,w} = B_c + \sum_{i=0}^{C_{\rm in}-1} \sum_{j=0}^{F-1}\sum_{k=0}^{F-1} w_{c,i,j,k}x_{i,j+h,k+w}, 
    \end{equation}
    where $B_c$ is a channel-wise bias.
    \item The 2-D transpose convolution of $y$ by $\tilde{w}$ is then defined in this work as the gradient of the 2-D convolution with respect to its input:
    \begin{equation}
        (\tilde{w} \star y) \defeq \frac{\partial (w \star x)}{\partial x}\cdot y
    \end{equation}
    \item The dot product ``$\bullet$'' generalized to pairs of tensors of same shape $(C,H,W)$:
    \begin{equation}
        a \bullet b = \sum_{c=0}^{C-1}\sum_{h=0}^{H-1}\sum_{w=0}^{W-1} a_{c,h,w}b_{c,h,w}.
    \end{equation}
    \item The pooling operation $\mathcal{P}$ with stride $F$ and filter size $F$ of $x$:
    \begin{equation}
        \mathcal{P}_{F}(x)_{c,h,w} = \underset{i,j \in [0,F-1]}{\rm max}  \left\{ x_{c, F(h-1)+1+i, F(w-1)+1+j} \right\},
    \end{equation}
    with relative indices of maximums within each pooling zone given by:
    \begin{equation}
        {\rm ind}_{\mathcal{P}}(x)_{c,h,w} = \underset{i,j \in [0,F-1]}{\rm argmax}  \left\{ x_{c, F(h-1)+1+i, F(w-1)+1+j} \right\} = (i^{*}(x,h), j^{*}(x,w)).
    \end{equation}
    \item The unpooling operation $\mathcal{P}^{-1}$ of $y$ with indices ${\rm ind}_{\mathcal{P}}(x)$ is then defined as:
    \begin{equation}
        \mathcal{P}^{-1}(y, {\rm ind}_{\mathcal{P}}(x))_{c,h,w} = \sum_{i,j} y_{c,i,j}\cdot \delta_{h, F(i-1)+1+i^{*}(x,h)} \cdot \delta_{w, F(j-1)+1+j^{*}(x,w)},
    \end{equation}
    which consists in filling a tensor with the same dimensions as $x$  with the values of $y$ at the indices ${\rm ind}_{\mathcal{P}}(x)$, and zeroes elsewhere. For notational convenience, we omit to write explicitly the dependence on the indices except when appropriate. 
    \item The flattening operation $\mathcal{F}$ is defined as reshaping a tensor of dimensions $(C,H,W)$ to $(1, CHW)$. We denote by $\mathcal{F}^{-1}$ its inverse.
\end{itemize}

\subsection{Convolutional RNNs with symmetric connections}
\label{sec:appConvSym}

In this section, we write explicitly the dynamics and the learning rules applied for the convolutional architecture with symmetric connections, for the Squared loss function and the Cross-Entropy loss function, for the one-sided and symmetric estimates.

\subsubsection{Squared Error loss}
\label{sec:appConvSym-SE}
\paragraph{Equations of the dynamics.} In this case, the dynamics read:

\begin{align}
\label{eq:conv-archi-sym}
\left\{
\begin{array}{l}
\displaystyle s^{n+1}_{t+1} = \sigma \left( \mathcal{P}(w_{n+1} \star s^{n}_{t}) + \tilde{w}_{n+2} \star \mathcal{P}^{-1}(s^{n+2}_{t}) \right), \qquad \forall n \in [0, N^{\rm conv}-2] \\
\displaystyle s^{N^{\rm conv}}_{t+1} = \sigma \left( \mathcal{P}(w_{N^{\rm conv}} \star s^{N^{\rm conv}-1}_{t}) + \mathcal{F}^{-1}({w_{N^{\rm conv}+1}}^{\top} \cdot s^{N^{\rm conv}+1}_{t}) \right), \\
\displaystyle s^{N^{\rm conv} + 1}_{t+1} = \sigma \left( w_{N^{\rm conv} + 1} \cdot \mathcal{F}(s^{N^{\rm conv}}_{t}) + {w_{N^{\rm conv} + 2}}^{\top} \cdot s^{N^{\rm conv} + 2}_{t} \right), \\
\displaystyle s_{t+1}^{n+1} = \sigma \left( w_{n+1} \cdot s^{n}_{t} + {w_{n+2}}^{\top} \cdot s^{n+2}_{t} \right), \qquad \forall n \in [N^{\rm conv} + 1, N^{\rm tot}-2] \\
\displaystyle s_{t+1}^{N^{\rm tot}} = \sigma \left( w_{N^{\rm tot}} \cdot s^{N^{\rm tot}-1}_{t}\right) + \beta(y - s^{N^{\rm tot}}), \quad \text{with $\beta=0$ during the first phase,}
\end{array}
\right.
\end{align}
where we take the convention $s^{0}=x$. In this case, we have $\hat{y} = s_{t+1}^{N^{\rm tot}}$.
Considering the function:

\begin{align*}
    \Phi(x, s^{1}, \cdots, s^{N^{\rm tot}}) &= 
    \sum_{n= N_{\rm conv} + 2}^{N_{\rm tot} - 1} s^{{n + 1}^\top}\cdot w_{n}\cdot s^{n}
    + s^{N_{\rm conv} + 1}\cdot w_{N_{\rm conv} + 1}\cdot \mathcal{F}(s_{t}^{N_{\rm conv}})\\
    &+\sum_{n = 1}^{N_{\rm conv} - 1} s^{n + 1}\bullet\mathcal{P}\left(w_{n + 1}\star s^{n}\right) + s^{1}\bullet\mathcal{P}\left(w_{1}\star x\right), 
\end{align*}

when ignoring the activation function, we have:

\begin{equation}
\label{eq:dphids-SE}
\forall n \in [1, N^{\rm tot}]: \quad s_t^n \approx  \frac{\partial \Phi}{\partial s^n}.
\end{equation}

Note that in the case of the Squared Error loss function, the dynamics of the output layer derive from $\Phi$ as it can be seen by Eq.~(\ref{eq:dphids-SE}). 

\paragraph{Learning rules for the one-sided EP estimator.} In this case, the learning rules read:

\begin{align}
   \left\{
\begin{array}{l}
\forall n \in [N_{\rm conv} + 2, N_{\rm tot} - 1]: \quad \Delta w_{n}  =
\frac{1}{\beta}\left(s_{*}^{n + 1, \beta}\cdot s_{*}^{{n, \beta}^\top} - s_{*}^{n + 1}\cdot s_{*}^{{n}^\top}  \right) \\
\Delta w_{N_{\rm conv}+ 1}  =
\frac{1}{\beta}\left(s_{*}^{N_{\rm conv} + 1, \beta}\cdot \mathcal{F}\left(s^{N_{\rm conv}, \beta}_{*}\right)^\top - s_{*}^{N_{\rm conv} + 1}\cdot \mathcal{F}\left(s^{N_{\rm conv}}_{*}\right)^\top \right) \\
\forall n \in [1, N_{\rm conv} - 1]:\quad \Delta w_{n + 1}  =  \frac{1}{\beta} \left(\mathcal{P}^{-1}(s^{n + 1, \beta}_{*})\star s^{n, \beta}_{*} - \mathcal{P}^{-1}(s^{n + 1}_{*})\star s^{n}_{*} \right)\\
\Delta w_1  =  \frac{1}{\beta} \left(\mathcal{P}^{-1}(s^{1, \beta}_{*})\star x - \mathcal{P}^{-1}(s^{1}_{*})\star x \right)  
\end{array}, 
\right. 
\label{deltaconv-sym-SE}
\end{align}

\paragraph{Learning rules for the symmmetric EP estimator.} In this case, the learning rules read:

\begin{align}
   \left\{
\begin{array}{l}
\forall n \in [N_{\rm conv} + 2, N_{\rm tot} - 1]: \quad \Delta w_{n}  =
\frac{1}{2\beta}\left(s_{*}^{n + 1, \beta}\cdot s_{*}^{{n, \beta}^\top} - s_{*}^{n + 1, -\beta}\cdot s_{*}^{{n, -\beta}^\top}  \right) \\
\Delta w_{N_{\rm conv}+ 1}  =
\frac{1}{2\beta}\left(s_{*}^{N_{\rm conv} + 1, \beta}\cdot \mathcal{F}\left(s^{N_{\rm conv}, \beta}_{*}\right)^\top - s_{*}^{N_{\rm conv} + 1, -\beta}\cdot \mathcal{F}\left(s^{N_{\rm conv}, -\beta}_{*}\right)^\top \right) \\
\forall n \in [1, N_{\rm conv} - 1]:\quad \Delta w_{n + 1}  =  \frac{1}{2\beta} \left(\mathcal{P}^{-1}(s^{n + 1, \beta}_{*})\star s^{n, \beta}_{*} - \mathcal{P}^{-1}(s^{n + 1, -\beta}_{*})\star s^{n, -\beta}_{*} \right)\\
\Delta w_1  =  \frac{1}{2\beta} \left(\mathcal{P}^{-1}(s^{1, \beta}_{*})\star x - \mathcal{P}^{-1}(s^{1, -\beta}_{*})\star x \right)  
\end{array}, 
\right. 
\label{deltaconv-sym-SE2}
\end{align}

\subsubsection{Cross-Entropy loss}
\label{sec:appConvSym-CE}
\paragraph{Equations of the dynamics.} In this case, the dynamics read:

\begin{align}
\label{eq:conv-archi-sym-softmax}
\left\{
\begin{array}{l}
\displaystyle s^{n+1}_{t+1} = \sigma \left( \mathcal{P}(w_{n+1} \star s^{n}_{t}) + \tilde{w}_{n+2} \star \mathcal{P}^{-1}(s^{n+2}_{t}) \right), \qquad \forall n \in [0, N^{\rm conv}-2] \\
\displaystyle s^{N^{\rm conv}}_{t+1} = \sigma \left( \mathcal{P}(w_{N^{\rm conv}} \star s^{N^{\rm conv}-1}_{t}) + \mathcal{F}^{-1}({w_{N^{\rm conv}+1}}^{\top} \cdot s^{N^{\rm conv}+1}_{t}) \right), \\
\displaystyle s^{N^{\rm conv} + 1}_{t+1} = \sigma \left( w_{N^{\rm conv} + 1} \cdot \mathcal{F}(s^{N^{\rm conv}}_{t}) + {w_{N^{\rm conv} + 2}}^{\top} \cdot s^{N^{\rm conv} + 2}_{t} \right), \\
\displaystyle s_{t+1}^{n+1} = \sigma \left( w_{n+1} \cdot s^{n}_{t} + {w_{n+2}}^{\top} \cdot s^{n+2}_{t} \right), \qquad \forall n \in [N^{\rm conv} + 1,N^{\rm tot}-3] \\
\displaystyle s_{t+1}^{N^{\rm tot}-1} = \sigma \left( w_{N^{\rm tot}-1} \cdot s^{N^{\rm tot}-2}_{t}\right) + \beta {w_{\rm out}}^{\top} \cdot (y - \hat{y}) \quad \text{with $\beta=0$ during the first phase,}\\
\displaystyle \hat{y} = {\rm softmax}(w_{\rm out} \cdot s^{N^{\rm tot}-1}_t),
\end{array}
\right.
\end{align}

where we keep again the convention $s^0 = x$. Considering the function:

\begin{align*}
    \Phi(x, s^{1}, \cdots, s^{N^{\rm tot} - 1}) &= 
    \sum_{n= N_{\rm conv} + 1}^{N_{\rm tot} - 2} s^{{n + 1}^\top}\cdot w_{n}\cdot s^{n}
    + s^{N_{\rm conv} + 1}\cdot w_{N_{\rm conv} + 1}\cdot \mathcal{F}(s_{t}^{N_{\rm conv}})\\
    &+\sum_{n = 1}^{N_{\rm conv} - 1} s^{n + 1}\bullet\mathcal{P}\left(w_{n + 1}\star s^{n}\right) + s^{1}\bullet\mathcal{P}\left(w_{1}\star x\right), 
\end{align*}

when ignoring the activation function, we have:

\begin{equation}
\label{eq:dphids-CE}
 \forall n \in [1, N^{\rm tot} - 1]: \quad s_t^n \approx  \frac{\partial \Phi}{\partial s^n}, \qquad  \hat{y} = {\rm softmax}(w_{\rm out} \cdot s^{N^{\rm tot}-1}_t). \\
\end{equation}

Note that in this case and contrary to the Squared Error loss function, the dynamics of the output layer do not derive from the primitive function $\Phi$, as it can be seen from Eq.~(\ref{eq:dphids-CE})

\paragraph{Learning rules for the one-sided EP estimator.} In this case, the learning rules read:

\begin{align}
   \left\{
\begin{array}{l}
\Delta w_{\rm out} = -\left( \widehat{y}_*^\beta - y \right)\cdot s_*^{\beta, N^\top}. \\
\forall n \in [N_{\rm conv} + 2, N_{\rm tot} - 2]: \quad \Delta w_{n}  =
\frac{1}{\beta}\left(s_{*}^{n + 1, \beta}\cdot s_{*}^{{n, \beta}^\top} - s_{*}^{n + 1}\cdot s_{*}^{{n}^\top}  \right) \\
\Delta w_{N_{\rm conv}+ 1}  =
\frac{1}{\beta}\left(s_{*}^{N_{\rm conv} + 1, \beta}\cdot \mathcal{F}\left(s^{N_{\rm conv}, \beta}_{*}\right)^\top - s_{*}^{N_{\rm conv} + 1}\cdot \mathcal{F}\left(s^{N_{\rm conv}}_{*}\right)^\top \right) \\
\forall n \in [1, N_{\rm conv} - 1]:\quad \Delta w_{n + 1}  =  \frac{1}{\beta} \left(\mathcal{P}^{-1}(s^{n + 1, \beta}_{*})\star s^{n, \beta}_{*} - \mathcal{P}^{-1}(s^{n + 1}_{*})\star s^{n}_{*} \right)\\
\Delta w_1  =  \frac{1}{\beta} \left(\mathcal{P}^{-1}(s^{1, \beta}_{*})\star x - \mathcal{P}^{-1}(s^{1}_{*})\star x \right)  
\end{array}.
\right. 
\label{deltaconv}
\end{align}

\paragraph{Learning rules for the symmetric EP estimator.} In this case, the learning rules read:

\begin{align}
   \left\{
\begin{array}{l}
\Delta w_{\rm out} = -\frac{1}{2}\left(\left( \widehat{y}_*^\beta - y \right)\cdot s_*^{\beta, N^\top} + \left( \widehat{y}_*^{-\beta} - y \right)\cdot s_*^{-\beta, N^\top}\right). \\
\forall n \in [N_{\rm conv} + 2, N_{\rm tot} - 2]: \quad \Delta w_{n}  =
\frac{1}{2\beta}\left(s_{*}^{n + 1, \beta}\cdot s_{*}^{{n, \beta}^\top} - s_{*}^{n + 1, -\beta}\cdot s_{*}^{{n, -\beta}^\top}  \right) \\
\Delta w_{N_{\rm conv}+ 1}  =
\frac{1}{2\beta}\left(s_{*}^{N_{\rm conv} + 1, \beta}\cdot \mathcal{F}\left(s^{N_{\rm conv}, \beta}_{*}\right)^\top - s_{*}^{N_{\rm conv} + 1, -\beta}\cdot \mathcal{F}\left(s^{N_{\rm conv, -\beta}}_{*}\right)^\top \right) \\
\forall n \in [1, N_{\rm conv} - 1]:\quad \Delta w_{n + 1}  =  \frac{1}{2\beta} \left(\mathcal{P}^{-1}(s^{n + 1, \beta}_{*})\star s^{n, \beta}_{*} - \mathcal{P}^{-1}(s^{n + 1, -\beta}_{*})\star s^{n, -\beta}_{*} \right)\\
\Delta w_1  =  \frac{1}{2\beta} \left(\mathcal{P}^{-1}(s^{1, \beta}_{*})\star x - \mathcal{P}^{-1}(s^{1, -\beta}_{*})\star x \right)  
\end{array}.
\right. 
\label{deltaconv2}
\end{align}

\subsubsection{Implementation details in PyTorch.} The equation of the dynamics as well as the EP estimates computation can be expressed as derivatives of the primitive function $\Phi$. 
Therefore, the automatic differentiation framework provided by PyTorch can be leveraged to implement implicitly the equations of the dynamics and the EP estimates computation by differentiating $\Phi$. 
Although this implementation is slower than explicitly implementing the equations of the dynamics, it is more flexible in terms of network architecture as $\Phi$ is relatively easy to compute.

\subsection{Convolutional RNNs with asymmetric connections}
\label{sec:appConvAsym}
In this section, we write the explicit definition of the dynamics and the learning rule of a convolutional architecture with asymmetric connections where forward and backward connections are no longer constrained to be equal-valued. In this setting, we use the Cross-Entropy loss function along with a softmax readout to implement the output layer of the network. 

\paragraph{Equations of the dynamics.}
In this setting, the dynamics Eq.~(\ref{eq:conv-archi-sym-softmax}) have simply to be changed into:

\begin{align}
\label{eq:conv-archi-asym-softmax}
\left\{
\begin{array}{l}
\displaystyle s^{n+1}_{t+1} = \sigma \left( \mathcal{P}(w^{\rm f}_{n+1} \star s^{n}_{t}) + \tilde{w}_{n+2}^{\rm b} \star \mathcal{P}^{-1}(s^{n+2}_{t}) \right), \qquad \forall n \in [0, N^{\rm conv}-2] \\
\displaystyle s^{N^{\rm conv}}_{t+1} = \sigma \left( \mathcal{P}(w^{\rm f}_{N^{\rm conv}} \star s^{N^{\rm conv}-1}_{t}) + \mathcal{F}^{-1}({w^{\rm b}_{N^{\rm conv}+1}}^{\top} \cdot s^{1}_{t}) \right), \\
\displaystyle s^{N^{\rm conv} + 1}_{t+1} = \sigma \left( w_{N^{\rm conv} + 1}^{\rm f} \cdot \mathcal{F}(s^{N^{\rm conv}}_{t}) + {w_{N^{\rm conv} + 2}}^{\rm b^{\top}} \cdot s^{N^{\rm conv} + 2}_{t} \right), \\
\displaystyle s_{t+1}^{n+1} = \sigma \left( w^{\rm f}_{n+1} \cdot s^{n}_{t} + {w^{\rm b}_{n+2}}^{\top} \cdot s^{n+2}_{t} \right), \qquad \forall n \in [N^{\rm conv} + 1, N^{\rm tot}-3] \\
\displaystyle s_{t+1}^{N^{\rm tot}-1} = \sigma \left( w^{\rm f}_{N^{\rm tot}-1} \cdot s^{N^{\rm tot}-2}_{t}\right) + \beta {w_{\rm out}}^{\top} \cdot (y - \hat{y}),\\
\displaystyle \hat{y} = {\rm softmax}(w_{\rm out} \cdot s^{N^{\rm tot}-1}_t),
\end{array}
\right.
\end{align}

where we distinguish now between forward and backward connections: $w^{\rm f}_n \neq w^{\rm b}_n \quad \forall n \in [1, N_{\rm tot} - 2]$. 

\paragraph{Original Vector Field learning rule (VF).}

The symmetric version of the original Vector Field learning rule is defined as:

\begin{equation}
\label{eq:sym-vf-estimate}
 \widehat{\nabla}^{\rm{VF}}_{\rm sym}(\beta) \defeq \frac{1}{2\beta} \frac{\partial F}{\partial \theta}(x, s_*, \theta)^\top \cdot \left(s^\beta_* - s^{-\beta}_*\right),
\end{equation}

which yields in the case of softmax read-out:

\begin{align}
   \left\{
\begin{array}{l}
\Delta w_{\rm out} = -\frac{1}{2}\left(\left( \widehat{y}_*^\beta - y \right)\cdot s_*^{\beta, N^\top} + \left( \widehat{y}_*^{-\beta} - y \right)\cdot s_*^{-\beta, N^\top}\right). \\
\forall n \in [N_{\rm conv} + 2, N_{\rm tot} - 2]: \quad \Delta w_{n}^{\rm f}  =
\frac{1}{2\beta} \left(s_{*}^{n + 1, \beta}-s_{*}^{n + 1, -\beta} \right)\cdot s_{*}^{{n}^\top} \\
\forall n \in [N_{\rm conv} + 2, N_{\rm tot} - 2]: \quad \Delta w_{n}^{\rm b}  =
\frac{1}{2\beta} s_{*}^{{n+1}} \cdot \left(s_{*}^{n, \beta}-s_{*}^{n, -\beta} \right)^{\top} \\
\Delta w_{N_{\rm conv}+ 1}^{\rm f}  =
\frac{1}{2\beta}\left(s_{*}^{N_{\rm conv} + 1, \beta} - s_{*}^{N_{\rm conv} + 1, -\beta}\right) \cdot  \mathcal{F}\left(s^{N_{\rm conv}}_{*}\right)^\top \\
\Delta w_{N_{\rm conv}+ 1}^{\rm b}  =
\frac{1}{2\beta} s_{*}^{N_{\rm conv} + 1} \cdot \left(\mathcal{F}\left(s^{N_{\rm conv}, \beta}_{*}\right) -  \mathcal{F}\left(s^{N_{\rm conv}, -\beta}_{*}\right) \right)^{\top}  \\
\forall n \in [1, N_{\rm conv} - 1]:\quad \Delta w_{n + 1}^{\rm f}  =  \frac{1}{2\beta} \left(\mathcal{P}^{-1}(s^{n + 1, \beta}_{*}) - \mathcal{P}^{-1}(s^{n + 1, -\beta}_{*}) \right) \star s^{n}_{*}\\
\forall n \in [1, N_{\rm conv} - 1]:\quad \Delta w_{n + 1}^{\rm b}  =  \frac{1}{2\beta} \mathcal{P}^{-1}(s^{n + 1}_{*})\star \left( s^{n, \beta}_{*} -  s^{n, -\beta}_{*} \right)\\
\Delta w_1  =  \frac{1}{2\beta} \left(\mathcal{P}^{-1}(s^{1, \beta}_{*})\star x - \mathcal{P}^{-1}(s^{1, -\beta}_{*})\star x \right)  
\end{array}.
\right. 
\label{deltaconvAsym}
\end{align}

Importantly, note that $\Delta w^{\rm f}_n \neq \Delta w^{\rm b}_n \quad \forall n \in [1, N_{\rm tot} - 2]$.

\paragraph{Kolen-Pollack algorithm.} When forward and backward weights have a common gradient estimate, and a weight decay term $\lambda$, they converge to the same values. We recall the proof, noting $t$ the iteration step and taking the notations of section \ref{sec:new-vf}, the update rule follows:

\begin{align*}
\left\{
\begin{array}{l}
\displaystyle \theta_{\rm f}(t+1) = \theta_{\rm f}(t) + \Delta \theta_{\rm f} \\
\displaystyle \theta_{\rm b}(t+1) = \theta_{\rm b}(t) + \Delta \theta_{\rm b}
\end{array}
\right..
\end{align*}

We can then write

\begin{align*}
    \theta_{\rm f}(t+1) - \theta_{\rm b}(t+1) &= \theta_{\rm f}(t) - \theta_{\rm b}(t) + \Delta \theta_{\rm f} - \Delta \theta_{\rm b} \\
    &= \theta_{\rm f}(t) - \theta_{\rm b}(t) - \eta \lambda \left( \theta_{\rm f}(t) - \theta_{\rm b}(t) \right) \\
    &= (1 - \eta \lambda) \left( \theta_{\rm f}(t) - \theta_{\rm b}(t) \right),
\end{align*}

where we use the fact that the estimates are the same for both parameters, such that they cancel out. Then by recursion :

\begin{align*}
    \theta_{\rm f}(t) - \theta_{\rm b}(t) &= (1 - \eta \lambda)^{t} \left( \theta_{\rm f}(0) - \theta_{\rm b}(0) \right) \underset{t \to \infty}{\rightarrow} 0, \quad \text{since} \quad |1 - \eta \lambda| < 1.
\end{align*}

\paragraph{Kolen-Pollack Vector Field learning rule (KP-VF).} We remind here that the new learning rule proposed in this paper to train convNets with asymmetric connections is defined as:

\begin{align}
\label{eq:KP-VF-update}
\left\{
\begin{array}{l}
\displaystyle \Delta \theta_{\rm f} = \eta \left(\widehat{\nabla}^{\rm{KP-VF}}_{\rm sym}(\beta) - \lambda \theta_{\rm f}\right)\\
\displaystyle \Delta \theta_{\rm b} = \eta \left( \widehat{\nabla}^{\rm{KP-VF}}_{\rm sym}(\beta) - \lambda \theta_{\rm b}\right)
\end{array}
\right.,
\quad \mbox{with} \quad \widehat{\nabla}^{\rm{KP-VF}}_{\rm sym}(\beta) = \frac{1}{2}(\overline{\nabla_{\theta_{\rm f}}^{\rm VF}}(\beta)+\overline{\nabla_{\theta_{\rm b}}^{\rm VF}}(\beta)),
\end{align}

where:

\begin{equation}
\label{eq:KP-VF}
\forall {\rm i \in \{ f,b \} }, \qquad \overline{\nabla_{\theta_{\rm i}}^{\rm VF}}(\beta) = \frac{1}{2 \beta} \left( \frac{\partial F}{\partial \theta_{\rm i}}^\top(x, s_*^\beta, \theta) \cdot s_*^\beta - \frac{\partial F}{\partial \theta_{\rm i}}^\top(x, s_*^{-\beta}, \theta)\cdot s_*^{-\beta} \right).
\end{equation}

More specifically, applying Eq.~(\ref{eq:KP-VF}) to Eq.~(\ref{eq:conv-archi-asym-softmax}) yields:

\begin{align}
   \left\{
\begin{array}{l}
\forall n \in [N_{\rm conv} + 2, N_{\rm tot} - 2]: \\ \qquad \qquad \qquad  \overline{\nabla_{w_{n}^{\rm f}}^{\rm VF}}(\beta) = \overline{\nabla_{w_{n}^{\rm b}}^{\rm VF}}(\beta) =
\frac{1}{2\beta} \left(s_{*}^{n + 1, \beta}\cdot s_{*}^{{n, \beta}^\top}-s_{*}^{n + 1, -\beta}\cdot s_{*}^{{n, -\beta}^\top} \right) \\
\overline{\nabla_{w_{N_{\rm conv}+ 1}^{\rm f}}^{\rm VF}}(\beta)= \overline{\nabla_{w_{N_{\rm conv}+ 1}^{\rm b}}^{\rm VF}}(\beta) = \\
\qquad \qquad \qquad  \frac{1}{2\beta}\left(s_{*}^{N_{\rm conv} + 1, \beta}\cdot  \mathcal{F}\left(s^{N_{\rm conv}, \beta}_{*}\right)^\top - s_{*}^{N_{\rm conv} + 1, -\beta}\cdot  \mathcal{F}\left(s^{N_{\rm conv}, -\beta}_{*}\right)^\top\right) \\
\forall n \in [1, N_{\rm conv} - 1]: \\ \overline{\nabla_{w_{n + 1}^{\rm f}}^{\rm VF}}(\beta)=  \frac{1}{2\beta} \left(\mathcal{P}^{-1}\left(s^{n + 1, \beta}_{*}, {\rm ind}_{\mathcal{P}}\left(w_{n + 1}^{\rm f}\star s^{n, \beta}_{*}\right)\right)\star s^{n, \beta}_{*}\right. \\
\qquad \qquad \qquad \qquad \qquad \qquad \qquad \left.- \mathcal{P}^{-1}\left(s^{n + 1, -\beta}_{*} , {\rm ind}_{\mathcal{P}}\left(w_{n + 1}^{\rm f}\star s^{n, -\beta}_{*}\right)\right)\star s^{n, -\beta}_{*} \right)\\
\forall n \in [1, N_{\rm conv} - 1]: \\ \overline{\nabla_{w_{n + 1}^{\rm b}}^{\rm VF}}(\beta) =  \frac{1}{2\beta} \left(\mathcal{P}^{-1}\left(s^{n + 1, \beta}_{*} , {\rm ind}_{\mathcal{P}}\left(w_{n + 1}^{\rm b}\star s^{n, \beta}_{*}\right)\right)\star s^{n, \beta}_{*} \right.\\
\qquad \qquad \qquad \qquad \qquad \qquad \qquad \left.  - \mathcal{P}^{-1}\left(s^{n + 1, -\beta}_{*}, {\rm ind}_{\mathcal{P}}\left(w_{n + 1}^{\rm b}\star s^{n, -\beta}_{*}\right)\right)\star s^{n, -\beta}_{*} \right)\\
\end{array}.
\right. 
\label{eq:nabla-KPVF}
\end{align}

Combining Eqs.~(\ref{eq:nabla-KPVF}) with Eq.~(\ref{eq:KP-VF-update}) gives the associated parameter updates. The updates for $w_1$ and $w_{\rm out}$ are the same than those of Eq.~(\ref{deltaconvAsym}).
Importantly, note that while $\forall n \in [N_{\rm conv} + 1, N_{\rm tot} - 2]: \quad  \overline{\nabla_{w_{n}^{\rm f}}^{\rm VF}}(\beta) = \overline{\nabla_{w_{n}^{\rm b}}^{\rm VF}}(\beta)$, we have $\forall n \in [1, N_{\rm conv} - 1]: \quad  \overline{\nabla_{w_{n}^{\rm f}}^{\rm VF}}(\beta) \neq \overline{\nabla_{w_{n}^{\rm b}}^{\rm VF}}(\beta)$ because of inverse pooling. In other words, the updates of the convolutional filters do not solely depend on the pre and post synaptic activations but also on the location of the maximal elements within each pooling window, itself depending on the filter considered. Hence the motivation to average $\overline{\nabla_{w_{n}^{\rm f}}^{\rm VF}}(\beta) $ and $\overline{\nabla_{w_{n}^{\rm b}}^{\rm VF}}(\beta)$ and use this quantity to update to $w_{n}^{\rm b}$ and $w_{n}^{\rm f}$ and apply the Kolen-Pollack technique. 

\paragraph{Implementation details in PyTorch.} The dynamics in the case of asymmetric connections does not derive from a primitive function $\Phi$. Therefore, it is not possible to implicitly get the dynamics by differentiating one primitive function.
A way around is to get the asymmetric dynamics by differentiating one quantity $\tilde{\Phi}^n$ by layer. This quantity is not a primitive function and is especially designed to get the right equations once differentiated. We define $\tilde{\Phi}^{n}(w_{n}^{\rm f}, w_{n+1}^{\rm b}, s^{n-1}, s^{n})$ by:

\begin{align}
   \left\{
\begin{array}{l}
\forall n \in [1, N_{\rm conv} - 1]: \tilde{\Phi}^n = s^n \bullet \mathcal{P}(w_{n}^{\rm f} \star s^{n - 1} ) + s^{n+1} \bullet \mathcal{P}(w_{n+1}^{\rm b}\star s^n) \\
\tilde{\Phi}^{N_{\rm conv}} = s^{N_{\rm conv}} \bullet \mathcal{P}(w_{N_{\rm conv}}^{\rm f} \star s^{N_{\rm conv} - 1} ) + s^{N_{\rm conv}+1} \cdot w_{N_{\rm conv}+1}^{\rm b}\cdot \mathcal{F}(s^{N_{\rm conv}}) \\
\forall n \in [N_{\rm conv}+1, N_{\rm tot} - 1]: \tilde{\Phi}^n = s^n \cdot w_{n}^{\rm f} \cdot s^{n - 1}  + s^{n+1} \cdot w_{n+1}^{\rm b}\cdot s^n \\
\tilde{\Phi}^{N_{\rm tot}-1} = s^{N_{\rm tot}-1} \cdot w_{N_{\rm tot}-1}^{\rm f} \cdot s^{N_{\rm tot} - 2} + \beta \ell( s^{N_{\rm tot}-1}, y, w_{\rm out})
\end{array},
\right.
\label{eq:cheatedPhi}
\end{align}
where $\ell$ is defined by Eq.~(\ref{eq:cross-entropy}), and $\beta=0$ in the first phase. Then, $\forall n \in [1, N^{\rm tot}-1]$, the dynamics of Eq.~(\ref{eq:conv-archi-asym-softmax}) read:

\begin{equation}
    s^{n}_{t+1} = \sigma \left( \frac{\partial \tilde{\Phi}^n}{\partial s^{n}}(w_{n}^{\rm f}, w_{n+1}^{\rm b}, s^{n-1}_{t}, s^{n}_{t})\right).
\end{equation}

The original VF update of Eq.~(\ref{deltaconvAsym}) can be written as $\forall n \in [1, N^{\rm tot}-1], \forall {\rm i} \in \{ {\rm f},{\rm b} \}$:

\begin{equation}
    \Delta w_{n}^{\rm i} = \frac{1}{2\beta}\left( \frac{\partial \tilde{\Phi}^n}{\partial w_{n}^{\rm i}}(s^{n, \beta}_*, s^{n-1}_*) - \frac{\partial \tilde{\Phi}^n}{\partial w_{n}^{\rm i}}(s^{n, -\beta}_*, s^{n-1}_*)\right), 
\end{equation}

and Eq.~(\ref{eq:nabla-KPVF}) as:

\begin{equation}
    \overline{\nabla_{w_{n}^{\rm i}}^{\rm VF}}(\beta) = \frac{1}{2\beta}\left( \frac{\partial \tilde{\Phi}^n}{\partial w_{n}^{\rm i}}(s^{n, \beta}_*, s^{n-1, \beta}_*) - \frac{\partial \tilde{\Phi}^n}{\partial w_{n}^{\rm i}}(s^{n, -\beta}_*, s^{n-1, -\beta}_*)\right). 
\end{equation}

\section{Experimental details}
\label{sec:appExpDetail}

\begin{figure}[ht!]
  \label{fig:mse}
  \centering
  \includegraphics[width=0.8\textwidth]{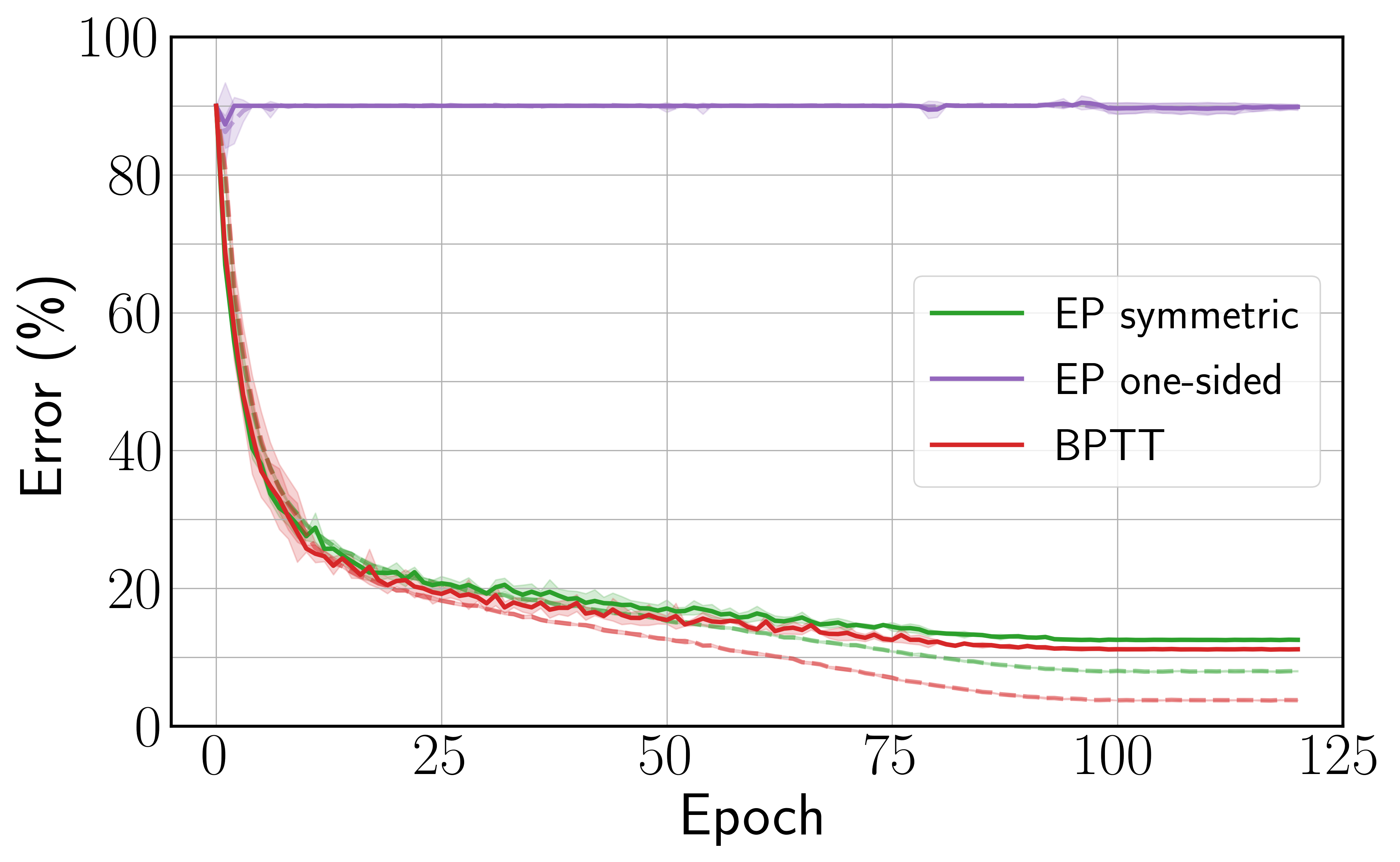}
  \caption{Train (dashed) and test (solid) errors on CIFAR-10 with the Squared Error loss function. The curves are averaged over 5 runs and shadows stand for $\pm 1$ standard deviation.}
\end{figure}

\begin{figure}[ht!]
  \label{fig:cel}
  \centering
  \includegraphics[width=0.8\textwidth]{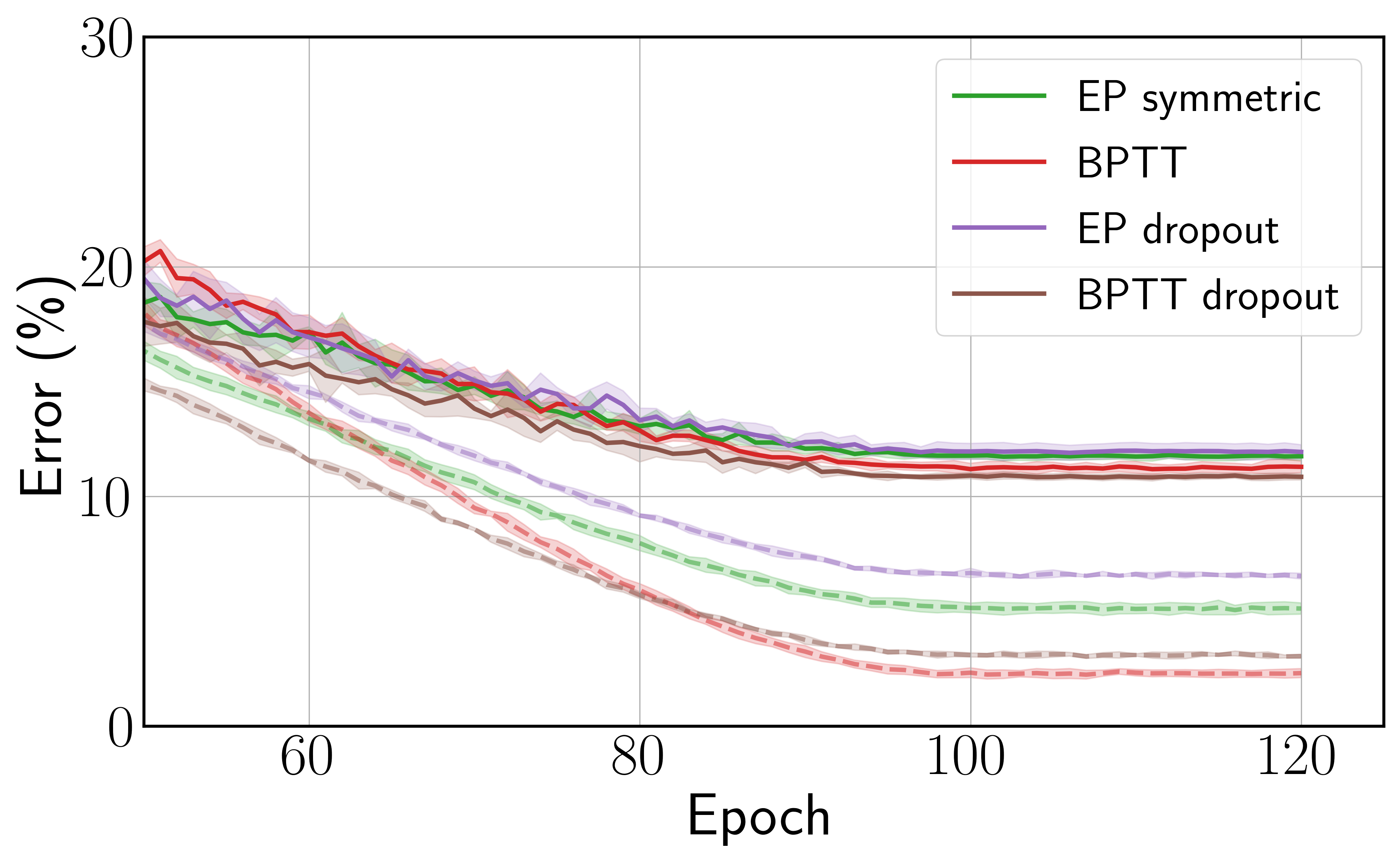}
  \caption{Train (dashed) and test (solid) errors on CIFAR-10 with the Cross-Entropy loss function. The curves are averaged over 5 runs and shadows stand for $\pm 1$ standard deviation.}
\end{figure}

\subsection{Environment and hyper-parameters}
The experiments are run using PyTorch 1.4.0 and torchvision 0.5.0. \citep{paszke2017automatic}. 
The convolutional architecture used in the CIFAR-10 experiment consists of four $3\times3$ convolutional layers of respective feature maps 128 - 256 - 512 - 512. 
We use a stride of one for each convolutional layer, and zero-padding of one for each layer except for the last layer. 
Each layer is followed by a $2\times2$ Max Pooling operation with a stride of two.
The resulting flattened feature vector is of size 512.
The weights are initialized using the default initialization of PyTorch, which is the uniform Kaiming initialization introduced by \citet{he2015delving}.
The data is normalized and augmented with random horizontal flips and random crops.
The training is performed with stochastic gradient descent with momentum and weight decay.
We use the learning rate scheduler introduced by \citet{loshchilov2016sgdr} to speed up convergence.
The simulations were carried across several servers consisting of 14 GPUs in total.
Each run was performed on a single GPU for an average run time of 2 days.

For the asymmetric architecture, the backward weights are defined for all convolutional layers except the first convolutional layer connected to the static input. 
The forward and backward weights are initialized independently at the beginning of training.
The backward weights have no bias contrary to their forward counterparts. 
The hyper-parameters such as learning rate, weight decay and momentum are shared between forward and backward weights.

\begin{table}[ht!]
\caption{Hyper-parameters used for the CIFAR-10 experiments.}
\label{tab:hyperparam}
\centering
\begin{tabular}{ccc}
\hline
Hyper-parameter                                                                  & Squared Error                             & Cross-Entropy                   \\ \hline
$T$                                                                             & 250                             & 250                             \\
$K$                                                                             & 30                              & 25                              \\
$\beta$                                                                         & 0.5                             & 1.0                             \\
Batch Size                                                                      & 128                             & 128                             \\
\begin{tabular}[c]{@{}c@{}}Initial learning rates\\ (Layer-wise)\end{tabular}   & 0.25 - 0.15 - 0.1 - 0.08 - 0.05 & 0.25 - 0.15 - 0.1 - 0.08 - 0.05 \\
Final learning rates                                                            & $10^{-5}$                       & $10^{-5}$                       \\
\begin{tabular}[c]{@{}c@{}}Weight decay\\ (All layers)\end{tabular}             & $3 \cdot 10^{-4}$               & $3 \cdot 10^{-4}$               \\
Momentum                                                                        & 0.9                             & 0.9                             \\
Epoch                                                                           & 120                             & 120                             \\
\begin{tabular}[c]{@{}c@{}}Cosine Annealing \\ Decay time (epochs)\end{tabular} & 100                             & 100                             \\ \hline
\end{tabular}
\end{table}

\subsection{Random-sign estimate variance}

The results presented in Table \ref{tab:results} consists of five runs. 
In the case of the EP random-sign estimate, one run among the five collapses to random guess similar to the one-sided estimate.
In order to test the frequency of such a phenomenon, we performed another five runs with both symmetric and random-sign estimates. 
The results for each run presented in Table \ref{tab:additionalRuns} show that two trials among ten are unstable, confirming further the high variance nature of the random-sign estimate.

\begin{figure}[ht!]
  \label{fig:rnd-sign}
  \centering
  \includegraphics[width=0.8\textwidth]{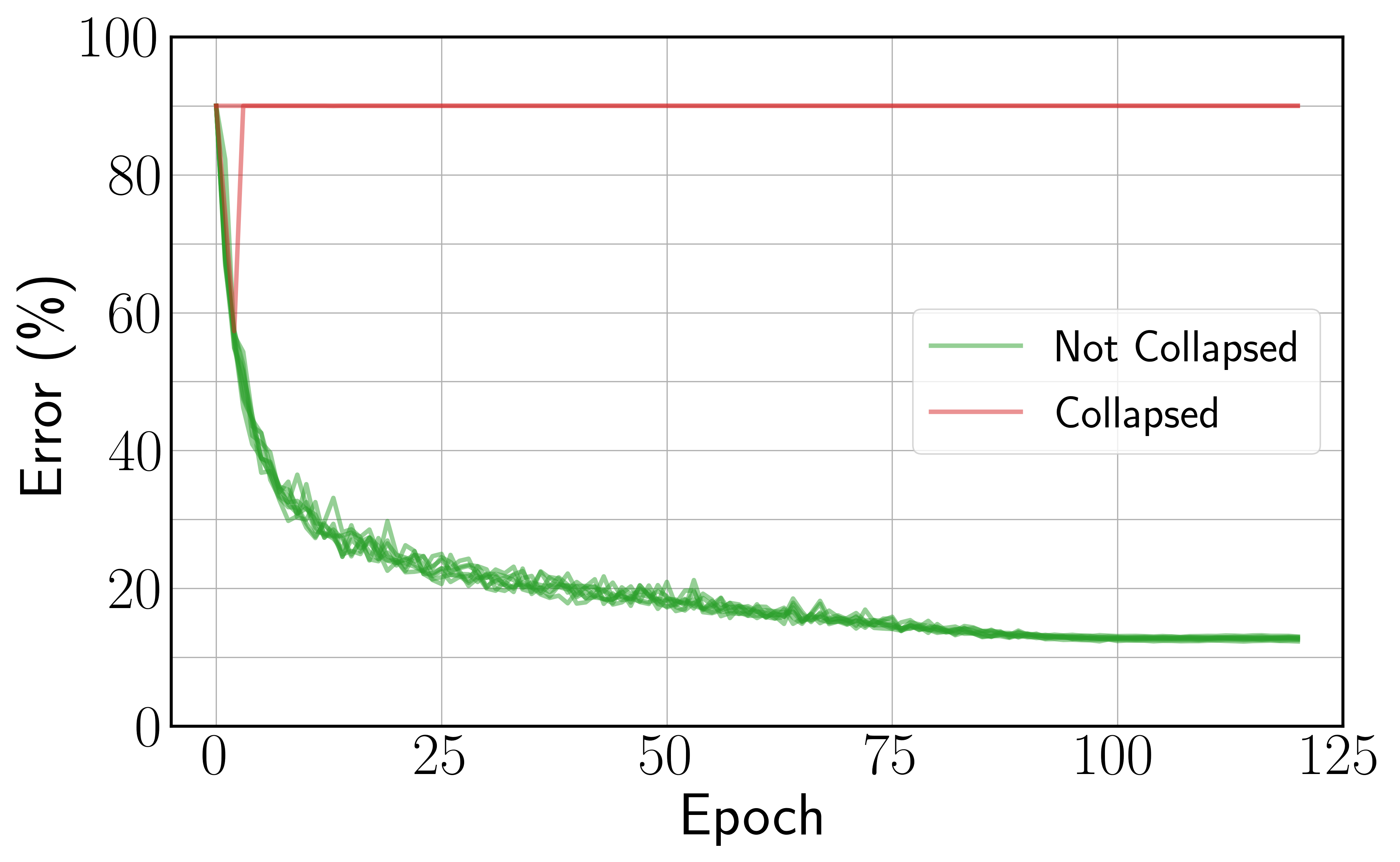}
  \caption{Test error curve of each run with the Squared Error loss function and random-sign estimate. The two collapsed runs among the ten trials are steady to 90\% because in such cases the network typically outputs the same class for each data point.}
\end{figure}

\begin{table}[ht!]
\caption{Best test error comparison between random-sign and symmetric estimates, for ten runs.}
\label{tab:additionalRuns}
\centering
\begin{tabular}{ccc}
\hline
Run index    & EP random-sign   & EP symmetric     \\ \hline
$1$          & $12.97$          & $12.24$          \\
$2$          & $12.72$          & $12.31$          \\
$3$          & $12.30$          & $12.68$          \\
$4$          & $12.45$          & $12.43$          \\
$5$          & $12.78$          & $12.57$          \\
$6$          & $12.66$          & $12.55$          \\
$7$          & $12.84$          & $12.44$          \\
$8$          & $12.59$          & $12.52$          \\
$9$          & $57.32$          & $12.85$          \\
$10$         & $89.98$          & $12.60$          \\ \hline
Mean         & $24.86$          & $\mathbf{12.52}$ \\ \hline
w/o collapse & $\mathbf{12.66}$ & N.A              \\ \hline
\end{tabular}
\end{table}

\subsection{Adding dropout}
\label{sec:dropout}
We adapt dropout \citep{srivastava2014dropout} for convergent RNNs by shutting some units to zero with probability $p<1$ when computing $\Phi(x, s_t, \theta)$. We multiply the remaining active units by the factor $\frac{1}{1-p}$ to keep the same neural activity on average, so that the learning rule is rescaled by $\left(\frac{1}{1-p}\right)^2$. 
The dropped out units are the same within one training iteration but they differ across the examples of one mini batch.
In our experiments we use $p=0.1$ on the last convolutional layer before the linear classifier.
The results are reported in Table~\ref{tab:results}.

\subsection{Changing the activation function}
\label{sec:appActivation}
Previous implementations of EP used a shifted hard sigmoid activation function:
\begin{equation}
    \sigma(x) = \max(0, \min(x, 1)).
    \label{eq:act-former}
\end{equation}
In their experiments with ConvNets on MNIST, \citet{ernoult2019updates} observed saturating units that cannot pass error signals during the second phase. 
In this work, to mitigate this effect, we have rescaled by a factor $1/2$ the slope of the activation function to ease signal propagation and prevent saturation, therefore changing Eq.~(\ref{eq:act-former}) into:
\begin{equation}
    \sigma(x) = \max\left(0, \min\left(\frac{x}{2}, 1\right)\right).
    \label{eq:act-now}
\end{equation}

\section{Weight alignment for asymmetric connections}
\label{sec:appAngle}
The angle $\alpha$ between forward and backward weights is defined as :
\begin{equation}
    \alpha = \frac{180}{\pi}{\rm Acos}\left(\frac{w^{\rm b} \bullet w^{\rm f}}{\|w^{\rm b}\| \|w^{\rm f}\|}\right), \quad \text{where} \quad \|w\| = \sqrt{w \bullet w}.
\end{equation}
Fig.~\ref{fig:angle} shows the angle between forward and backward weights during training on CIFAR-10 for both the original VF learning rule (dashed) and the new learning rule inspired by \citet{kolen1994backpropagation}. 

\begin{figure}[ht!]
  \centering
  \includegraphics[width=0.8\textwidth]{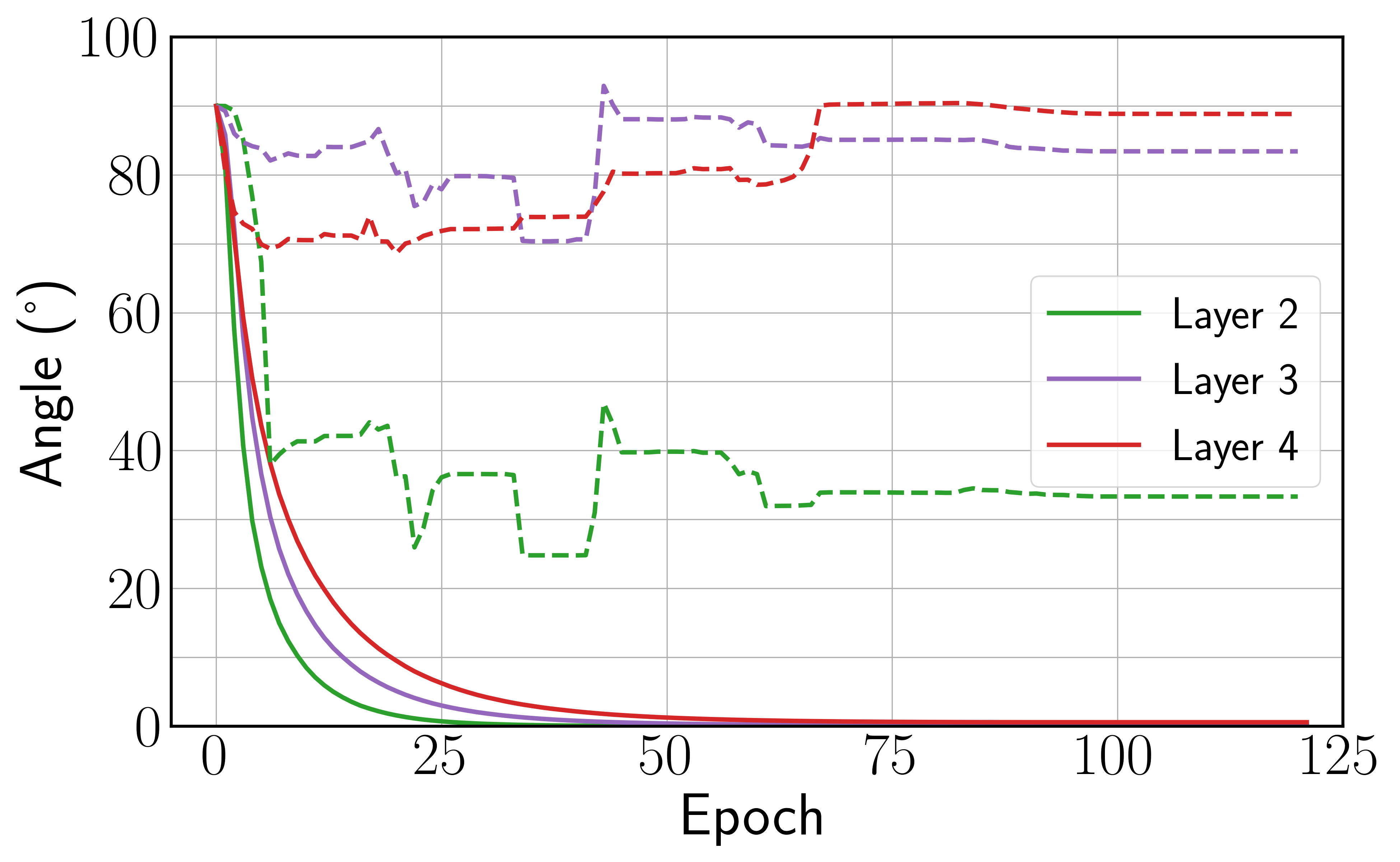}
  \caption{Angle $\alpha$ between forward and backward weights for the new estimate $\widehat{\nabla}^{\rm{KP-VF}}_{\rm sym}$ introduced (solid) and $\widehat{\nabla}^{\rm{VF}}_{\rm sym}$ (dashed).
  }
  \label{fig:angle}
\end{figure}

\section{Layer-wise comparison of EP estimates}
\label{sec:appCompEstimate}

In this section we show on Fig.~\ref{fig:morecurve} more instances of Fig.~\ref{fig:thirdphase} for each layer of the convolutional architecture.

\begin{figure}[ht!]
  \centering
  \includegraphics[width=\textwidth]{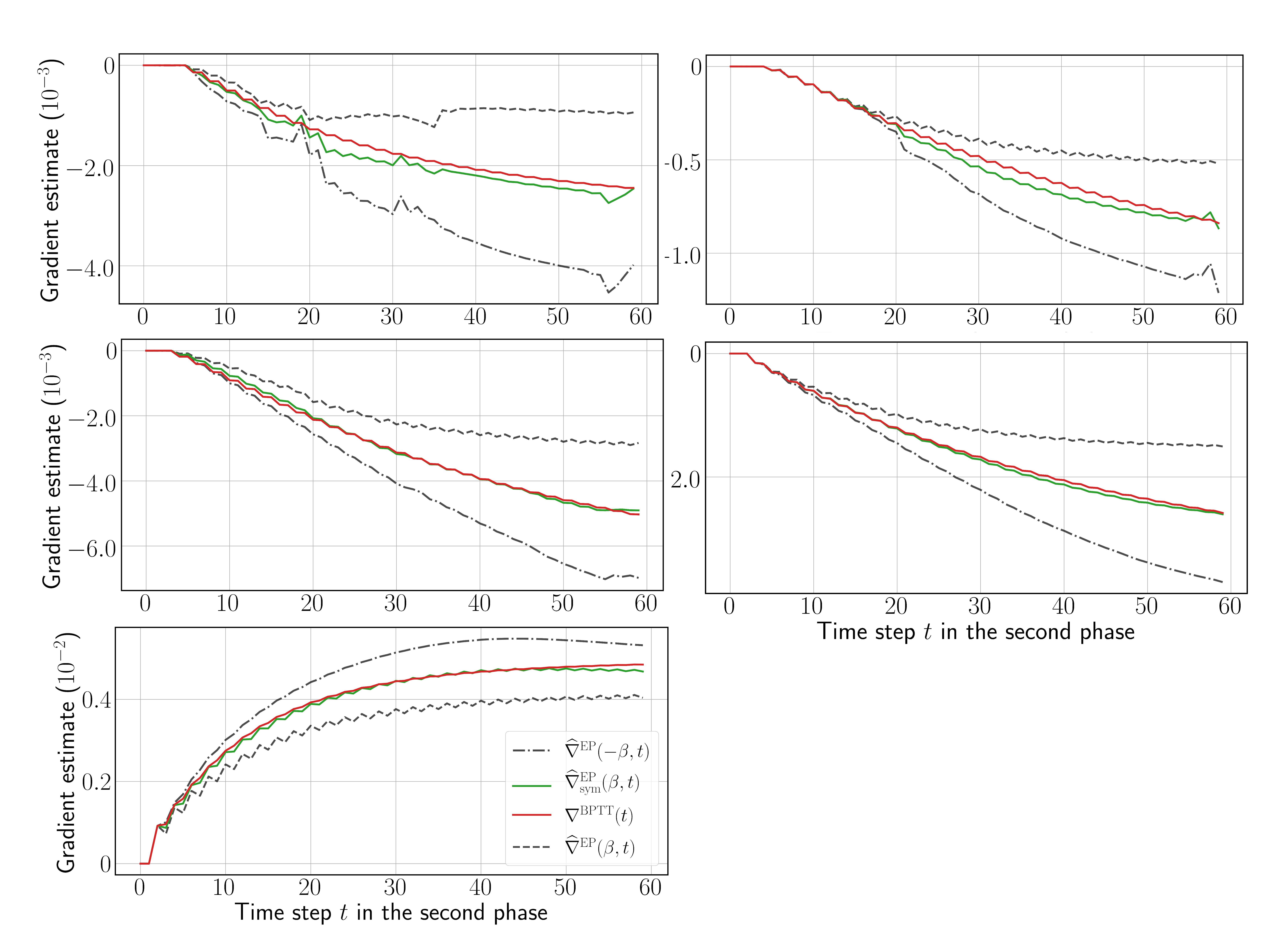}
  \caption{Layer-wise comparison between EP gradient estimates and BPTT gradients for 5 layers deep CNN on CIFAR-10 Data. Layer index increases from top to bottom, left to right, top-left being the first layer.
  }
  \label{fig:morecurve}
\end{figure}

\end{document}